\documentclass[12pt]{article}
\usepackage{amsmath}
\usepackage{graphicx}
\usepackage{enumerate}
\usepackage{natbib}
\usepackage{url} % not crucial - just used below for the URL 

\usepackage{subfigure}
\usepackage{graphicx}
\usepackage{multirow} % Include the multirow package
\usepackage{bbm}

\usepackage{booktabs}
\usepackage{titletoc}
\usepackage{titlesec}

\usepackage{mathrsfs}
\usepackage{amsmath,amssymb}
\usepackage{bm}
\usepackage{natbib}
\usepackage[usenames]{color}
\usepackage{amsthm}
\usepackage{arydshln}

\usepackage{multirow} 
\usepackage{enumitem}
\usepackage{caption}
\usepackage{subcaption}
\usepackage{enumitem}
\usepackage{dsfont}
\usepackage{mathtools}
\usepackage{caption}
\usepackage{booktabs}

\usepackage{hyperref}
\usepackage{mylatexstyle}

\usepackage{cleveref}
\newcommand{\blind}{1}

\usepackage{xcolor}

\ifdefined\final
\usepackage[disable]{todonotes}
\else
\usepackage[textsize=tiny]{todonotes}
\fi
\allowdisplaybreaks

\begin{document}

\def\spacingset#1{\renewcommand{\baselinestretch}%
{#1}\small\normalsize} \spacingset{1}

%%%%%%%%%%%%%%%%%%%%%%%%%%%%%%%%%%%%%%%%%%%%%%%%%%%%%%%%%%%%%%%%%%%%%%%%%%%%%%

\if1\blind
{
  % \title{\bf Transformers Learn Heterogeneous Bayesian Networks: Inferring Architectures via Pre-Training, and Estimating Parameters from Context}
  \title{\bf Transformers Simulate MLE for Sequence Generation in Bayesian Networks}
  \author{\hspace{.2cm} Yuan Cao\thanks{Equal Contribution.}~\thanks{The University of Hong Kong, \url{yuancao@hku.hk}},\quad
    Yihan He\footnotemark[1]~\thanks{Princeton University, \url{{yihan.he, jqfan}@princeton.edu}},\quad 
    Dennis Wu\thanks{Northwestern University, \url{{yu-hsuanwu2024,hong-yuchen2029,hanliu}@northwestern.edu} },\quad Hong-Yu Chen\footnotemark[4],\\\hspace{.2cm}
    Jianqing Fan\footnotemark[3],\quad and\quad Han Liu\footnotemark[4]} 
  \date{}
  \maketitle
} \fi

\if0\blind
{
  \bigskip
  \bigskip
  \bigskip
  \begin{center}
    {\LARGE\bf }
\end{center}
  \medskip
} \fi

\bigskip
\begin{abstract}
Transformers have achieved significant success in various fields, notably excelling in tasks involving sequential data like natural language processing. Despite these achievements, the theoretical understanding of transformers' capabilities remains limited. In this paper, we investigate the theoretical capabilities of transformers to autoregressively generate sequences in Bayesian networks based on in-context maximum likelihood estimation (MLE). Specifically, we consider a setting where a context is formed by a set of independent sequences generated according to a Bayesian network. We demonstrate that there exists a simple transformer model that can (i) estimate the conditional probabilities of the Bayesian network according to the context, and (ii) autoregressively generate a new sample according to the Bayesian network with estimated conditional probabilities. We further demonstrate in extensive experiments that such a transformer does not only exist in theory, but can also be effectively obtained through training. Our analysis highlights the potential of transformers to learn complex probabilistic models and contributes to a better understanding of large language models as a powerful class of sequence generators.
\end{abstract}

\noindent%
\vfill

\newpage
\spacingset{1.9} % DON'T change the spacing!

% especially in learning sequential data such as natural language processing tasks.

 % We further demonstrate that the conditional probability estimation and the autoregressive generation by the transformer model are both optimal.

% the goal is to generate a new sample according to the 

% (i) estimate the conditional probabilities of the Bayesian network according to the context, and (ii) autoregressively generate a new sample according to the Bayesian network with estimated conditional probabilities. 

% autoregressively generate a new sample from the Bayesian network based on the 

% another partially observed sample of the ba

% we aim to use a transformer model

% we show that, for any Bayesian network architecture, there exists a two-layer transformer model that 

% can optimally estimate the conditional probabilities associated to the Bayesian network according to observed sequences of variables in the context. Moreover, this transformer also serves as an autoregressive genera

% generating samples based on partially observed variables in the query in an autoregressive manner.

% understanding by leveraging their ability to capture long-range dependencies and contextual information. 
\section{Introduction}
Transformers \citep{vaswani2017attention} have achieved tremendous success across various fields. These models are known to be particularly strong in terms of sequence generation, and have revolutionized the way we approach problems related to text generation, translation, and scientific discoveries such as protein generation. Despite these achievements, there remains limited understanding of the theoretical capabilities of transformers as sequence generators.

% have revolutionized the way we approach problems related to text generation, translation, as well as scientific discoveries such as protein generation. Despite these achievements, there remains limited understanding of the theoretical capabilities of transformers as a sequence generator.

% have revolutionized the way we approach problems related to text generation, translation, and scientific discoveries such as protein generation. Despite these achievements, there remains limited understanding of the theoretical capabilities of transformers as sequence generators.

To theoretically understand how transformers efficiently generate sequences, several recent works have studied the the power of transformers in learning specific probability models for sequential data \citep{ildizself,rajaramantransformers,makkuva2024attention,nichani2024transformers,edelman2024evolution,chen2024unveiling}. Specifically, \citet{ildizself} studied the problem of learning Markov chains with a one-layer self-attention model, and developed identifiability and convergence guarantees under certain conditions. \citet{rajaramantransformers} studied the behavior of transformers on data drawn from 
$k$-gram Markov processes, where the conditional distribution of the next variable in a sequence depends on the previous $k$ variables, and showed that such processes can be learned well by transformers of a constant-order depth. \citet{makkuva2024attention} further studied the loss function landscape of one-layer transformers in learning Markov chains. \citet{edelman2024evolution} empirically studied the training dynamics of two-layer transformer models in learning bigram Markov chains in context, and discussed how the results generalize to learning $n$-gram Markov chains. \citet{chen2024unveiling} established theoretical  guarantees on how two-layer transformers can be trained by gradient flow to perform in-context learning on $n$-gram Markov chain data. 
\citet{nichani2024transformers} studied a setting where the tokens consist of multiple sequences of samples generated from a causal network, and demonstrated that transformers can be trained to learn the causal network structure so that, when seeing a new context-query pair, it can generate prediction according to the learned causal structure and the context. However, \citet{nichani2024transformers} mostly focused on the setting where each variable has at most one parent.
% \citet{nichani2024transformers} 

% Another closely related line of works \citep{akyurek2022learning,zhang2023trained,bai2023transformers,huang2023context}  studied the power of transformers in solving \text{in-context learning} tasks. Specifically, \citet{akyurek2022learning,zhang2023trained,bai2023transformers,huang2023context} theoretically studied how transformers can perform in-context linear regression under the setting that the context consists of a training data set and the query token contains a test data for prediction. 

% More recently, \citet{nichani2024transformers} studied how transformers can learn causal structures. Specifically, under the assumption that the tokens consist of multiple sequences of samples generated from a causal network, \citet{nichani2024transformers} demonstrated that gradient descent can pre-train a transformer to learn the causal network structure.
% By doing so, when transformer sees a new context-query pair, it can generate prediction according to the learned network structure and the context. However, the analysis in \citet{nichani2024transformers} was mostly limited to the setting where each variable has at most one parent.

% In this paper, we aim to theoretically investigate the capability of transformers to autoregressively learn Bayesian networks in-context. 

In this work, we aim to give an in-depth analysis of transformers in sequence generation. 
% Instead of considering particular models such as Markov chains or ``single-parent'' probability models, 
Specifically, we consider the setting where the relationship among the sequential variables is characterized by a \textit{Bayesian network}, which covers the Markov chains and the ``single-parent'' probability models consider in previous works \citep{ildizself,rajaramantransformers,makkuva2024attention,nichani2024transformers} as special cases. In addition, inspired by the setting in \citet{nichani2024transformers} as well as the recent studies of ``in-context'' learning capabilities of transformers \citep{akyurek2022learning,zhang2023trained,bai2023transformers,huang2023context}, we also consider the case of in-context maximum likelihood estimation. Instead of acting as a simple generator of a fixed distribution, we require the transformer to adapt to new contexts. To fulfill this task, the transformer model must take a 'contextual dataset' as its input, \textit{perform the MLE algorithm} on the input dataset, and then base the sequence generation on the result of the MLE. Despite this complex setup, our analysis demonstrates that simple transformer models are capable of performing this task. 
% Under this challenging setting, we demonstrate that transformers can perform maximum likelihood estimation (MLE) in-context, and then autoregressively generate new sequences accordingly. 
% focus on a specific setting where a collection of independent samples generated from a Bayesian network  are observed and form a context, and our goal is to investigate the capability of  transformers to autoregressively learn Bayesian networks in-context. 
The main contributions of this paper are two-fold: providing clean and intuitive theoretical analyses, and presenting robust experimental studies. Specifically, our contributions can be summarized as follows.
\begin{itemize}[leftmargin = *]
    \item Theoretically, we demonstrate the existence of a transformer model that is capable of: (i) performing MLE for the conditional probabilities of the Bayesian network given the context, and (ii) autoregressively generating a new sequence based on these estimated conditional probabilities. This gives an intuitive demonstration on the capability of transformers to perform complicated sequence generation tasks.
    \item Empirically, we perform extensive experiments to validate our theoretical claims. Specifically, under various settings where the Bayesian network is a (Markov) chain, a tree, or a general graph, we demonstrate that a transformer can indeed be pre-trained from scratch, so that it can perform in-context estimations of conditional probabilities, and help sample a new sequence of variables accordingly. We also present real-data experiment results to further back up our conclusion in more practical settings.  
\end{itemize}

\textbf{Notations.} We use lowercase letters to denote scalars and boldface lowercase/uppercase letters to denote vectors/matrices, respectively. For a matrix $\Ab$, we use $\| \Ab \|_2$ to denote its spectral norm. For an integer $n$, we denote $[n] = { 1, 2, \ldots, n }$. For a set $S$, we use $|S|$ to denote its cardinality. We also use $\ind[\cdot]$ to denote an indicator function that equals $1$ when the corresponding statement is true and equals $0$ otherwise.

% \citep{cao2019generalization}

% ICL survey: \citep{dong2022survey}
% ICL theory: \citep{von2023transformers, nichani2024transformers, shen2023pretrained, ahn2024transformers, li2023transformers, wies2024learnability}
% ICL empirical analysis: \citep{garg2022can, wei2023larger, zhang2023trained, zhou2023algorithms, grazzi2024mamba, park2024can}

\section{Related Work}
\noindent\textbf{Transformers.}
Transformers \cite{vaswani2017attention} and its variants have demonstrated its success in various of domains such as language \cite{devlin2018bert, liu2019roberta, raffel2020exploring, touvron2023llama, achiam2023gpt}, vision \cite{dosovitskiy2020image, jia2022visual, liu2021swin, peebles2023scalable}, multi-modality \cite{gal2022image, radford2021learning, li2023blip} etc.
Large language models (LLMs) demonstrate remarkable ability to learn tasks in-context during inference, bypassing the need to update parameters \cite{brown2020language, lampinen2022can, khandelwal2018sharp}.
However, the understanding of the inner mechanisms of these models, and how they perform such complex reasoning tasks largely remain undiscovered \citep{dong2022survey}.
Such disadvantage prevents us to interprete why transformers often struggles to generalize well under out-of-distribution scenarios, especially on simple reasoning and logical tasks such as arithmetic \citep{magister2022teaching, touvron2023llama, ebrahimi2020can, suzgun2022challenging}.
This raise a doubt on how and when can transformers learn the appropriate algorithms to solve tasks or not.

\noindent\textbf{In-Context Learning.}
Recently, a line of work studies transformers through the lens of in-context learning (ICL), the ability of models to generate predictions based on a series of examples.
Empirically, recent studies find out transformers are capable of learning a series of functions in-context \citep{garg2022can, wei2023larger, zhang2023trained, zhou2023algorithms, grazzi2024mamba, park2024can, akyurek2022learning},
showing transformers can learn to approximate a wide range of algorithms.
Theoretically, \citet{akyurek2022learning,zhang2023trained,huang2023context} studied how transformers can perform in-context linear regression under the setting that the context consists of a training data set and the query token contains a test data for prediction. 
Several works also analyze the algorithmic approximation perspective of transformers under various of conditions \citep{von2023transformers, nichani2024transformers, shen2023pretrained, ahn2024transformers, li2023transformers, wies2024learnability}.
A recent work \cite{von2023transformers} shows that linear transformers \citep{katharopoulos2020transformers} are capable of performing gradient descent based on in-context examples.
In \citep{bai2023transformers}, they not only show ReLU transformers are capable of approximating gradient descent with small error, but can also capable of implementing more complex ICL processes involving \textit{in-context algorithm selection}.
To the best of authors' knowledge there is no existing literature that theoretically and empirically shows transformers learn to perform maximum likelihood estimation in-context for Bayesian network data.

% \section{Learning Bayesian networks with Transformers}

\section{Problem Setup}
In this section, we introduce the sequence generation task we consider, and discuss how we consider using a transformer model to handle this task.

\subsection{Sequence Generation and Bayesian Networks}
The specific sequence generation task we consider can be formulated by Bayesian networks. A Bayesian network is a probabilistic graphical model which specifies the conditional dependencies among the variables by a directed acyclic graph. Each node of the Bayesian network represents for a random variable, and the edges connected to a node indicates the ``parent(s)'' and ``child(ren)'' of the node. Furthermore, Bayesian networks modeling discrete random variables can be parameterized by parameters that form \textit{conditional probability tables}, which define the conditional distribution of each random variable given its parent(s). 

% briefly review the definitions of Bayesian network models, and formally define the problem we consider.

% \textbf{The sequence generation task.} 

Suppose that $X_1, \ldots, X_M$ are a sequence of $M$ discrete random variables following a certain distribution. It is a classic result that, there always exists a Bayesian network modeling the joint distribution of $X_1, \ldots, X_M$ such that $X_{1}$ is a 'root' variable with no parents, and for any $i \in [M]$, the parents of $X_{i}$ are all among $X_{1}, \ldots, X_{i-1}$. In addition, there exists a unique Bayesian network satisfying these properties, in which each variable has the smallest number of parents. We denote this Bayesian network as $\cB$. 
Our goal is to generate a new sequence of realizations of $X_1, \ldots, X_M$ according to $\cB$. However, we suppose that the conditional probability tables, i.e., the parameters of $\cB$, are unknown. Instead, we are given $N$ independent groups of observations $X_{1i}, \ldots. X_{Mi}$, $i=1,\ldots, N$ generated according to $\cB$.

\subsection{MLE and Autogregressive Generation by Transformers}

% In this work, we consider the problem of learning the conditional probability tables of such a Bayesian network $\cB$ in-context. Specifically, suppose that we are given $N$ groups of context observations $X_{1i}, \ldots. X_{Mi}$, $i=1,\ldots, N$ that are independently generated according to $\cB$. Then, 

We study the capability of transformers to autoregressively sample a new sequence  $X_{1q}, \ldots. X_{Mq}$ based on conditional probability tables estimated from the context.

% learn the conditional probability tables and to sample new variables accordingly. 

Suppose that the discrete random variables $X_1, \ldots. X_M$ takes $d$ possible values. 
 % $m$-th variable in the $i$-th observation
For $i\in[N]$ and $m\in[M]$, denote by $\xb_{mi}$ the $d$-dimensional one-hot vector of the observation $X_{mi}$. 
Moreover, suppose that at a certain step during the autoregressive generation process, some variables among $X_1, \ldots, X_M$ have been generated, and the goal is to generate the next variable. We define the query sequence $\xb_{1q},\ldots, \xb_{Mq}$ as follows: 
% $\xb_{mq}$ denotes the value of the $m$-th variable in the query: 
\begin{itemize}[leftmargin = *]
    \item If $X_{mq}$ is already sampled, then $\xb_{mq}$ is the one-hot vector representing the obtained value.
    \item If $X_{mq}$ is not sampled, then $\xb_{mq}$ is a zero vector.
\end{itemize}
Suppose that at the current step, the target is to sample $X_{m_0 q}$. We define  additional vectors
\begin{align}\label{eq:pos_embedding}
    \pb = [\mathbf{0}_{d(m_0-1)}^\top, \mathbf{1}_d^\top, \mathbf{0}_{d(M - m_0+1)}^\top ]^\top,~ \pb_q = [\mathbf{0}_{d(m_0-1)}^\top, \mathbf{1}_d^\top, \mathbf{0}_{d(M - m_0)}^\top, \mathbf{1}_d ]^\top \in \RR^{(M+1)d}.
\end{align}
The definition of $\pb$ and $\pb_q$ serves two purposes. First of all, they can teach an autoregressive model the current variable-of-interest. Moreover, the difference bewteen $\pb$ and $\pb_q$ also serves as an indicator of the ``query'' variable in the input. Based on these definitions, we define 
\begin{align}\label{eq:input_mat}
    \Xb = \begin{bmatrix}
        \xb_{11} & \xb_{12} & \cdots & \xb_{1N} & \xb_{1q}  \\
        \xb_{21} & \xb_{22} & \cdots & \xb_{2N} & \xb_{2q}\\
        \vdots & \vdots & & \vdots & \vdots\\
        \xb_{M1} & \xb_{M2} & \cdots & \xb_{MN}  & \xb_{Mq}\\
        \pb & \pb & \cdots & \pb &\mathbf{p}_{q}
    \end{bmatrix},
\end{align}
The matrix $\Xb$ can then be directly fed into a transformer model whose output aims to give the estimated distribution of $X_{m_0 q}$ as a $d$-dimensional vector that sums to one. If such a transformer model exists, then the autoregressive sampling process can be achieved according to Algorithm~\ref{alg:sampling}. The major goal of this paper is to investigate whether transformers can handle such tasks well. 

\begin{algorithm}[ht]
\caption{Autoregressive Sampling}
\begin{algorithmic}[1]\label{alg:sampling}
    \STATE \textbf{input:} Observations $\{\xb_{mi}: m\in [M], i\in [N]\}$, model  $\fb:\RR^{(2M+1)d \times (N+1)} \rightarrow \RR^{d}$.
    \STATE Initialize $\xb_{mq} = \mathbf{0}_{d}$ for $m\in[M]$. 
    \FOR {$m_0=1$ \textbf{to} $ M$}
        \STATE Set $\pb$ and  $\pb_q$ according to \eqref{eq:pos_embedding}, and define $\Xb$ according to \eqref{eq:input_mat}. 
        % \STATE Define $\Xb$ according to \eqref{eq:input_mat}.
        \STATE Sample $X_{m_0 q}$ according to $\fb(\Xb)$, and update $\xb_{m_0 q}$ as the corresponding one-hot vector.
    \ENDFOR
\end{algorithmic}
\end{algorithm}

\textbf{Maximum likelihood estimation of conditional probabilities.} To measure the performance of transformers, we consider comparing the output of the transformer with the optimal conditional distribution estimation given by maximizing the likelihood. For discrete random variables, it is will-known that the maximum likelihood estimation is obtained by \textit{frequency counting}. Specifically, suppose that at a certain step in the autoregressive sampling procedure, the model is aiming to sample the $m_0$-th variable. Denote by $\cP(m_0)$ the set consisting of the indices of the parents of $X_{m_0}$. Then, the sampling probability vector $\pb^{\mathrm{MLE}}_{m_0}\in \RR^d$  given by MLE is
\begin{align*}
    [\pb^{\mathrm{MLE}}_{m_0}]_j = \frac{ | \{ i\in [N]: X_{m_0i} = j, \text{ and } X_{mi} = X_{mq} \text{ for all } m\in \cP(m_0) \} | }{ | \{ i\in [N]: X_{mi} = X_{mq} \text{ for all } m\in \cP(m_0) \} | }.
\end{align*}
Further by the fact that $\xb_{mi}$'s and  $\xb_{mq}$'s are one-hot vectors, we can also write
\begin{align*}
    \pb^{\mathrm{MLE}}_{m_0} = \sum_{i\in[N]}  \frac{ \ind[  \xb_{mi} = \xb_{mq} \text{ for all }m\in \cP(m_0) ] }{| \{ i\in [N]: \xb_{mi} = \xb_{mq} \text{ for all }m\in \cP(m_0) \}|} \cdot \xb_{m_0i}.
\end{align*}
To compare a function output $\fb\in \RR^d$ with the MLE solution above, we consider the total variation distance between the two corresponding distributions. Specifically, if $\fb$ is a distribution vector (i.e., $\fb \in \RR_+^d$ and $\sum_{i=1}^d \fb_i = 1$), then we define
\begin{align*}
    \mathrm{TV}(\fb, \pb^{\mathrm{MLE}}_{m_0} ) : = \frac{1}{2}\sum_{j=1}^d | [\fb]_j - [\pb^{\mathrm{MLE}}_{m_0}]_j|.
\end{align*}
% the model  
% $\pb_q = [\mathbf{0}_{d(m_0-1)}^\top, \mathbf{1}_d^\top, \mathbf{0}_{d(M - m_0)}^\top ]^\top \in \RR^{Md}$ that teaches an autoregressive model the current variable-of-interest. Based on these definitions, we define 

% This vector $\pb_q$ is then attached to all the tokens defined by 

% serves as a positional embedding. 
%$\pb_q = [\mathbf{0}_d^\top, , \ldots, \mathbf{0}_d^\top, \mathbf{1}_d^\top, \mathbf{0}_d^\top, \ldots \mathbf{0}_d^\top]^\top \in \RR^{Md\times 1}$

% $m\in [M]$ and $i\in[N]$,

% where $\mathbf{p}_{q} = [\mathbf{0}_d^\top, , \ldots, \mathbf{0}_d^\top, \mathbf{1}_d^\top, \mathbf{0}_d^\top, \ldots \mathbf{0}_d^\top]^\top \in \RR^{Md\times 1}$ is a position embedding where  

% By gradually adding more variables to the model starting from $X_1$,

% a Bayesian network that models the joint distribution of $M$ random variables $X_1, X_2,\ldots. X_M$. Without loss of generality, we assume that $X_1, X_2,\ldots. X_M$ are ordered such that the 

% the joint distribution of multiple random variables based on a Given 

\section{Main theory}
% \section{Transformers learn Bayesian networks in-context}
% In this section, we present our main result.
We consider standard transformer architectures introduced in \citet{vaswani2017attention} that consists of self-attention layers and feed-forward layers with skip connections. Specifically, in our setup, an attention layer with parameter matrices $\Vb\in \RR^{(2M+1)d\times (2M+1)d}, \Kb \in \RR^{Md\times (2M+1)d}, \Qb \in \RR^{Md\times (2M+1)d}$ is defined as follows:
\begin{align}\label{eqn:attn}
    \mathrm{Attn}_{\Vb, \Kb, \Qb} (\Xb) = \Xb + \Vb \Xb \mathrm{softmax}[ (\Kb \Xb)^\top  (\Qb \Xb)  ],
\end{align}
where $\mathrm{softmax}$ denotes the column-wise softmax function. Here we also consider skip connections, which are commonly implemented in practice. In addition, a feed-forward layer with skip connections and parameter matrices $\Wb_1,\Wb_2 \in \RR^{(2M+1)d\times (2M+1)d}$ is defined as follows:
\begin{align}\label{eqn:ffl}
    \mathrm{FF}_{\Wb_1,\Wb_2} ( \Xb ) = \Xb + \Wb_2 \sigma( \Wb_1 \Xb ), 
\end{align}
where $\sigma(\cdot)$ denotes the entry-wise activation function. We consider the ReLU activation function $\sigma(z) = \max\{0,z\}$. Given the above definitions, we follow the convention in \citet{bai2023transformers} and call the following mapping a ``transformer layer'':
\begin{align*}
    \mathrm{TF}_{\btheta}( \Xb ) = \mathrm{FF}_{\Wb_1,\Wb_2} [ \mathrm{Attn}_{\Vb, \Kb, \Qb} (\Xb) ],
\end{align*}
where $\btheta = (\Vb,\Kb, \Qb,\Wb_1,\Wb_2 )$ denotes the collection of all parameters in the self-attention and feed-forward layer.

The above specifies the definition of a transformer layer, which is a mapping from $\RR^{(2M+1)d\times (N+1)}$ to $\RR^{(2M+1)d\times (N+1)}$ (for any $N\in \NN_+$). To handle the task of generating $d$-class categorical variables, we also need to specific the output of the model, which maps matrices in $\RR^{(2M+1)d\times (N+1)}$ to vectors in $\RR^{d}$. Here we follow the  common practice, and define the following $\mathrm{Read}(\cdot)$ function
\begin{align}\label{eqn:read}
    \mathrm{Read}(\Zb) := \Zb \eb_{N+1} \text{ for all } \Zb\in \RR^{(2M+1)d\times (N+1)}
\end{align}
to output the last column of the input matrix, and consider a linear mapping $\mathrm{Linear}_{\Ab}(\cdot)$ to convert the output of the  $\mathrm{Read}(\cdot)$ function to the final distribution vector:
\begin{align*}
    \mathrm{Linear}_{\Ab}(\zb) = \Ab\zb \text{ for all } \zb \in \RR^{(2M+1)d},
\end{align*}
where $\Ab \in \RR^{d\times(2M+1)d }$ is the paramter matrix of the linear mapping. 
% Following common practice, we also consider 

Given the above definitions, we are ready to introduce our main theoretical results, which are summarized in the following theorem.

\begin{theorem}\label{thm:main}
    For any $\epsilon > 0$, and any Bayesian network $\cB$ with maximum in-degree $D$, there exists a two-layer transformer model
    \begin{align*}
    \fb(\Xb) =  \mathrm{Linear}_{\Ab} [\mathrm{Read} ( \mathrm{TF}_{\btheta^{(2)}}( \mathrm{TF}_{\btheta^{(1)}}( \Xb ) ) ) ]
    \end{align*}
    with parameters satisfying
    % $\|\Vb^{(1)} \|_2, \|\Kb^{(1)} \|_2, \|\Qb^{(1)} \|_2, \|\Wb_2^{(1)} \|_2 \leq 1$ and $ \|\Wb_1^{(1)} \|_2\leq 2\sqrt{D+1}$
    % $\|\Vb^{(2)} \|_2, \|\Wb_1^{(2)} \|_2, \|\Wb_2^{(2)} \|_2  \leq 1$ and $ \|\Kb^{(2)} \|_2, \|\Qb^{(2)} \|_2\leq 3\log(MdN/\epsilon)$,
    \begin{align*}
        &\|\Vb^{(1)} \|_2, \|\Kb^{(1)} \|_2, \|\Qb^{(1)} \|_2, \|\Wb_2^{(1)} \|_2, \|\Vb^{(2)} \|_2, \|\Wb_1^{(2)} \|_2, \|\Wb_2^{(2)} \|_2, \|\Ab\|_2 \leq 1,\\
        &\|\Wb_1^{(1)} \|_2\leq 2\sqrt{D+1},  \|\Kb^{(2)} \|_2, \|\Qb^{(2)} \|_2\leq 3\log(MdN/\epsilon),
    \end{align*}
    % and a linear mapping satisfying $\|\Ab\|_2 \leq 1$,
    such that for any $m_0\in [M]$ and  $\pb$, $\pb_q$ defined according to $m_0$, it holds that $\fb(\Xb)$ is a probability vector, and
    \begin{align*}
        \mathrm{TV}\{ \fb(\Xb) , \pb_{m_0}^{\mathrm{MLE}} \} \leq \epsilon.
    \end{align*}
\end{theorem}
Theorem~\ref{thm:main} shows that there exists a two-layer transformer with an appropriate linear prediction layer such that, for any variable of interest $X_{m_0}$, the transformer can always output a distribution vector that is close to the maximum likelihood estimation $\pb_{m_0}^{\mathrm{MLE}}$ in total variation distance. Importantly, for Bayesian networks with bounded maximum in-degrees, the transformer we demonstrate has weight matrices with bounded (up to logarithmic factors) spectral norms, showcasing that despite the complex nature of the task, it can be handled well by transformers with ``bounded complexities''.  This provides strong evidence of the efficiency  of transformers in learning Bayesian networks in-context.

% \citep{bai2023transformers}

A notable pattern of the result in Theorem~\ref{thm:main} is that it demonstrates the capability of transformers to generate a sequence of variables in an autoregressive manner -- the parameters of the transformer do not depend on the index of the variable of interest $m_0$, and the same transformer model works for all $m_0\in [M]$ as long as the vectors $\pb$, $\pb_q$ appropriately defined according to $m_0$. This means that, the transformer model $ \fb(\Xb)$ can be utilized in the autoregressive sampling procedure in Algorithm~\ref{alg:sampling}, such that at each step, the transformer always sample the corresponding variable with close-to-optimal distributions.
% \begin{align*}
%     \fb(\Xb) =  \mathrm{Linear}_{\Ab} [\mathrm{Read} ( \mathrm{TF}_{\btheta^{(2)}}( \mathrm{TF}_{\btheta^{(1)}}( \Xb ) ) ) ]
% \end{align*}

% \begin{lemma}
%     There exists a self-attention layer
%     \begin{align*}
%      \Xb \mathrm{softmax} (\mathrm{Mask}(\Xb^\top \Wb\Xb)) 
%     \end{align*}
%     that fills the unobserved root entries in $\xb_{\cdot q}$ with the frequency vectors, while all the other entries in $\Xb$ are unchanged. 
% \end{lemma}

\section{Experiments}
The main paper contains four parts of the experiments.
First, we verify our theoretical results by studying the capabilities of transformers learning Bayesian networks on synthetic datasets.
Second, we analyze whether trained transformers are capable of generalizing to different value of $N$.
Next, we perform an analysis on whether our theoretical construction is optimal.
Last, we conduct a case study on the ACSIncome dataset.
In the appendix, we show the impact of different parameters on model performance.

\subsection{Transformers Perform MLE based on Bayesian Network Architecture}\label{exp:learn-bi}
Here we conduct the experiment of training transformers to perform MLE based on Bayesian Network Architecture.
We also visualize their convergence result with loss and accuracy curves.
\begin{figure}[t]
    \centering
\includegraphics[width=0.95\linewidth]{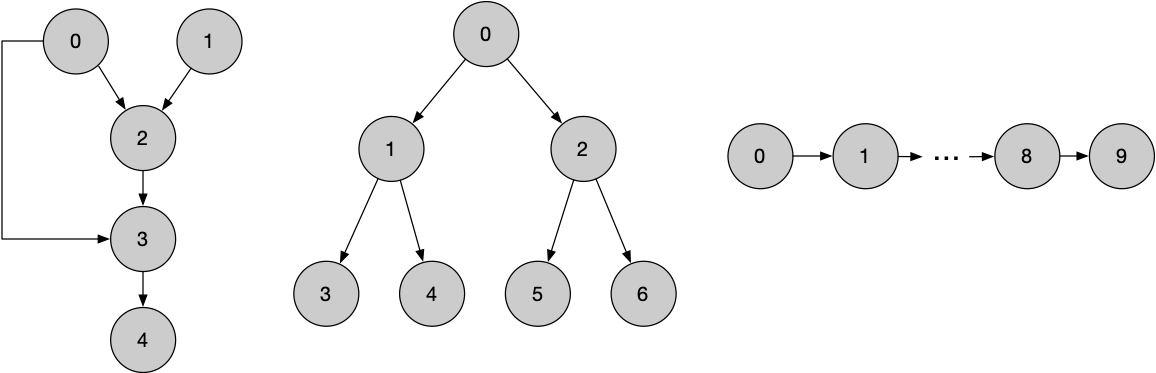}
    \caption{\textbf{Illustrations of graph structures in the experiments.}
    Left to right:
    general graph, tree and chain.
    The curriculum follows the number order of variables.
    The arrow indicates the causal relationships between variables.
    % Note that for general graph, variable 2, 3 both have 2 parents.
    % However, modeling variable 0, 1 is identical for naive inference and MLE based on true network, and is \textbf{NOT} for variable 3.
    % For tree, modeling root and variable 1 is identical for naive inference and MLE based on true network.
    % For chain, modeling variable 0, 1 is identical for naive inference and MLE based on Bayesian network architecture.
    }
    \label{fig:graphs-visualization}
\end{figure}

\textbf{Datasets.}
We consider training transformers to learn Bayesian networks of three structures:
chain, tree and general graph, see Figure~\ref{fig:graphs-visualization} for illustration.
All variables in our dataset are with binary values (2 possible outcomes).
For each structure, we generate 50k graphs with randomly initialized probability distributions, and sample all training data from them.
We process the data as following:
Given a Bayesian network structure $\mathcal{B}$, we first randomly select the $m$-th variable to predict.
Next, we randomly generate the parameters of $\mathcal{B}$ (values of conditional probabilities) and sample $N+1$ observations of the first $m$ variables. 
For the $N+1$-th observation, we mask out its $m$-th variable with $0$, and treat the value of this variable as the label of this sample. 
Finally, we encode the $N+1$ samples as in \eqref{eq:input_mat} by concatenating them into a $N+1$ column matrix and adding positional encoding vectors.
% To reduce noise during training, 
To randomly generate the values of conditional probabilities, we evenly sample them from one of the two uniform distributions $U(0.15, 0.3),\; U(0.7, 0.85)$.
By considering such distributions, we aim to avoid two scenarios: 
(i) random guess is already almost `optimal' (happens when the conditional probability is close to 0.5) and 
(ii) variables yield a deterministic relationship (happens when the conditional probabilities of a variable is close to 0 or 1). We sample $S = 64k, 19.2k, 12.8k$ independent contexts for chain, tree and general graphs, respectively. Within each context, we further sample $N+1$ independent observations with $N = 100$.

\textbf{Model.} 
We use a 6-layer transformer as our model.
In each layer, the transformer consists of a feed-forward layer (Equation~\eqref{eqn:ffl}) following by an attention layer (Equation~\eqref{eqn:attn}) and a layer normalization \cite{ba2016layer}.
Each attention layer has 8 heads with hidden dimension of 256, and each feed-forward layer has the hidden dimension of 1024.
We use a readout layer (Equation~\eqref{eqn:read}) to map the output of the transformer to the final distribution of prediction.
The parameters of the transformer is trained via Adam \cite{kingma2014adam}.
For prediction, we use a Softmax function to convert the output of the readout layer into a probability distribution over all possible outcomes (2-dimension if the variable is binary).
A small difference to our theoretical construction is that we use even simpler positional embeddings: we set $\pb\in \RR^{M}$ as zero vectors, and $\pb_q \in \RR^{M}$ a one-hot vector, indicating the variable to predict. 
For Bayesian networks with only binary variables, he input dimension is always 3 times of the number of variables in the graph as our construction in Equation~\ref{eq:input_mat}.
% \CC{(Yuan: We need to present the detailed setup, instead of saying its similar to Bai et al. You can refer to previous equations to explain the transformer model we use in experiments. We need to assume that the reviewers know nothing about transformers at all, except what we have mentioned in our paper.)}
% We follow most settings in \cite{bai2023transformers} by using a transformer with 8 heads and 256 hidden dimension for all experiments, the only difference is we use 6 instead of 12 layer transformers.
% A small difference is we set $\bp\in \RR^{M}$ as zero vectors, and $\bp_q \in \RR^{M}$ a one-hot vector, indicating the variable to predict.
% Note that we remove the positional embedding/encoding from the transformer to highlight the utility of $\bp$ \CC{(Yuan: this sentence is not very clear. We shuold avoid mentioning it. $\bp$ is our position embedding vector.)}.
% The input dimension is always 3 times of the number of variables in the graph as our construction in Equation~\ref{eq:input_mat}.
% We also observe the training, test accuracy throughout the training process.
% Additionally, we observe the confidence difference between the optimal predictor (using the true probability table and graph structure), and the confidence generated by the transformers.

\textbf{Training setup.}
% \CC{(Yuan: again, avoiding mentioning Bai et al. Please explain that we follow the common transformer training methods based on next token prediction, and give formal definitions of the loss function and give more details about the curriculum. I think we may need multiple paragraphs to explain everything, but that is fine. We also need to mention the optimization algorithm we use.)}
We train the transformer on the next token prediction task, a classic training procedure for transformers.
Given a input sequence of length $N$, the transformer predicts the outcome of the $N+1$-th token.
We assume the outcome distribution is discrete and therefore treat it as a classification task.
Let the trained transformer be $\hat{\fb}(\cdot)$.
Given any input $\Xb$, we use a Softmax function to normalize its output $\hat{\fb}(\Xb)$ into a probability distribution over all possible outcomes, i.e.,
\begin{align}\label{eq:transformer-output}
\hat{\yb}
=
\text{Softmax}
\left(
\hat{\fb}(\Xb)
\right),
\end{align}
where $\hat{\yb}$ is a probability distribution over $d$ possible outcomes.
We use $\yb$ to be the true label of $\Xb$, which is a $d$-dimensional one-hot vector indicating the true outcome.
Now following the notation in the dataset paragraph, let the training data contains $S$ input sequences of $N+1$ and $\yb_s$ be the true outcome of the $m$-th variable of the $s$-th input, the transformer is trained the minimize the CrossEntropy loss defined as following
\[
L(\yb, \hat{\yb}) =
-
\frac{1}{S}
\sum_{s=1}^S
 \sum_{i=1}^d [\yb_s]_i \log([\hat{\yb_s}]_i),
\]
where $[\yb_s]_i$ is the $i$-th element of $\yb_s$ and $[\hat{\yb_s}]_i$ is the $i$-th element of $\hat{\yb_s}$

% We process the data as following:
% Given a Bayesian Network structure $\mathcal{B}$, we first randomly select the $m$-th variable to predict.
% Next, we randomly initialize the parameters of $\mathcal{B}$ and sample $N+1$ observations on variables with order less or equal to $m$ ($1 \cdots m$).
% For the $N+1$-th observation, we mask out its $m$-th variable with $0$, and treat as the label of this sample.
% Finally, we encode our $N+1$ samples as Equation~\eqref{eq:input_mat} by concatenating them into a $N+1$ column matrix and adding positional encoding vectors.

We train the transformer with $10k, 3k, 2k$ steps on chain, tree and general graph, respectively.
For each step, the transformer takes a mini-batch of size $64$ as input, and updates its learnable parameters with sample-wise average loss within the mini-batch.
Each mini-batch has $64$ different contexts.
We use the AdamW \citep{loshchilov2017decoupled} optimizer with different learning rate based on the network structure (See Table~\ref{table:hyperparam}).
We use the notation of $N_{\text{train}}$ and $N_{\text{test}}$ as the $N$ used in training and testing, respectively.
We set $N_{\text{train}}=100$ during training, and vary $N_{\text{test}}$ when testing.
% The transformer is trained to minimize the CrossEntropy Loss defined as following.
% Let the trained transformer be $\fb(\cdot)$, we denote the probability distribution output by 
% Let $\hat{y}$ be the output of the transformer, and $y$ be the true label,
% we train the model to minimize the following objective function
% \begin{align*}
% L(y, \hat{y}) = - \sum_{i=1}^d y_i \log(\hat{y}_i),
% \end{align*}
% where $d$ denotes the number of possible outcomes of the label (the $m$-th variable).

% \textbf{Curriculum Design.} 
% \CC{(This part should  be merged to training setup.)}
We take the data-level curriculum approach to train the transformers performing MLE based on Bayesian Network Architecture.
The goal of the curriculum is to lead transformers to learn the whole graph structure well.
We determine the difficulty of the curriculum by the number of variables in the graph.
Therefore, we design the curriculum from easy to hard by revealing more and more variables throughout training.
By doing so, the graph structure ``grows" during training.
We start by revealing only the first two variables in the graph, meaning the transformer will only learn to predict the first 2 variables.
After the training loss reaches to a threshold, we then advance the curriculum by revealing one extra variable.
% Note that after revealing one extra variable, the training data now contains the observation of all revealed variables.

\textbf{Metrics.} 
% \CC{(Yuan: we need to add explanations of the  baseline methods, and explanations about how to calculate test accuracy, like how do we treat the transformer and the baseline methods as prediction models.)}
We denote the number of examples during training as $N_{\text{train}}$, and $N_{\text{test}}$ as the number of examples during evaluation.
For evaluations, we randomly generate 1 graph for each graph structure as testset, denoting as $\cB_{\text{test}}$.
We report the accuracy of transformers, Naive Inference, MLE based on true Bayesian network (we use MLE to represent it in figures), and the optimal accuracy on testset and vary the number of examples in each prediction.

Here we explain the metrics of all baselines and the transformer used in our experiments.
We use an example for predicting the $m_0$-th variable of a query sequence with the first to $(m_0-1)$-th variables $X_{1q}, \dots, X_{(M-1)q}$ being observed.
Following the setup of in-context learning, we assume a set of $N$ groups context observations $X_{1i}, \dots, X_{Mi}$ for $i = 1, \dots, N$.
Both the baselines and the transformer outputs a probability distribution over all possible outcomes.
We then select the outcome with the highest probability as their prediction.
Note that the two baselines are not capable of handling unseen features or labels.
Such a case will lead directly to assigning probability $0$ to all categories.
For the transformer, we obtain its prediction on the $s$-th sample by applying the argmax function on $\hat{\yb}_s$ defined in Equation~\eqref{eq:transformer-output}.
% on 
% to the output of the final readout layer (Equation~\eqref{eqn:read}) as the following
% \[
% P(\xb_{Mq} | \xb_{1q}, \dots, \xb_{(M-1)q}, \; \mathcal{O}, \; \bm{\theta} )
% =
% \text{Softmax}
% \left(
% f(\Xb)
% \right).
% \]
The naive inference method predicts $\xb_{m_0q}$ with the following probability distribution.
\[
\PP(X_{m_0q} | X_{1q}, \dots, X_{(m_0-1)q}) 
= 
\frac{\sum_{i =1}^N \mathbbm{1} \left(X_{1i} = X_{1q}, \dots, X_{(m_0-1)i} = X_{(m_0-1)q} , X_{m_0} = X_{m_0q}\right) }{\sum_{i=1}^N \mathbbm{1} \left(X_{1i} = X_{1q}, \dots, X_{(m_0-1)i} = X_{(m_0-1)q} \right)}.
\]
% where $\mathbbm{1}$ is the indicator function. 
In other words, the naive inference method performs MLE assuming a ``fully-connected'' Bayesian network.

% \paragraph{MLE based on True BN.}
For the last baseline, we perform MLE based on true Bayesian network structure, which we assume the network structure is known. 
% Thus, assuming the parents of the $M$-th variable are in the set of $\mathcal{P}$, where $\mathcal{P}_q$ are the parent nodes of $\xb_{Mq}$,
Specifically, the MLE method predicts $X_{Mq}$ with the following probability distribution.
% \[
% P(\xb_{Mq} | \xb_{1q}, \dots, \xb_{(M-1)q}, \; \mathcal{O}) 
% = 
% \frac{\sum_{X \in \mathcal{O}} \mathbbm{1} \left(X_1 = \xb_{1q}, \dots, X_{M-1} = \xb_{(M-1)q} \right) \cdot \mathbbm{1}(X_{M} = \xb_{Mq})}{\sum_{X \in \mathcal{O}} \mathbbm{1} \left(X_1 = \xb_{1q}, \dots, X_{M-1} = \xb_{(M-1)q} \right)},
% \]
% \[
% P(\xb_{Mq} | \mathcal{P}_q, \; \mathcal{O}) 
% = \frac{\sum_{(X_i, \mathcal{P}_i) \in \mathcal{O}} \mathbbm{1}(\mathcal{P}_i = \mathcal{P}_q) \cdot \mathbbm{1}(X_{M} = \xb_{Mq})}{\sum_{(X_i, \mathcal{P}_i) \in \mathcal{O}} \mathbbm{1}(\mathcal{P}_i = \mathcal{P}_q)},
% \]
% where $\mathcal{P}_i$ is the parent of $X_{Mi}$.
% \[
% P(X_{Mq} | X_{1q}, \dots, X_{(M-1)q}) 
% = 
% \frac{\sum_{i =1}^N \mathbbm{1} \left(X_{mi} = X_{mq}, i\in \cP(m_0)\right) \cdot \mathbbm{1}(X_{M} = X_{Mq})}{\sum_{i=1}^N \mathbbm{1} \left(X_{1i} = X_{1q}, \dots, X_{(M-1)i} = X_{(M-1)q} \right)}.
% \]
\[
\PP(X_{m_0q} | X_{1q}, \dots, X_{(m_0-1)q}) 
= 
\frac{\sum_{i =1}^N \mathbbm{1} \left(X_{mi} = X_{mq}, m\in\cP(m_0)\cup\{m_0\}\right) }{\sum_{i =1}^N \mathbbm{1} \left(X_{mi} = X_{mq}, m\in\cP(m_0)\right)}.
\]
Finally, the optimal accuracy is based on the prediction using the ground truth parameter of $\cB_{\text{test}}$ without the use of any examples.

Note that the optimal accuracy is not 1 due to the probabilistic nature of networks.
For each number of examples $N_{\text{test}} \in [5, 100]$, we randomly sample a set of 1500 observations, with each observation contains $N_{\text{test}}$ ICL examples and 1 test token.
We separate the evaluation of each variable in the graph as they have different optimal accuracy.
The reported accuracy are the average over 10 runs with different random seeds.
Due to space limit, we select 3 variables for each graph structure to present.
For tree, we select one variable for each level from root to leaf.
For chain, we select three variables that are close to the beginning, middle and the end of the chain.
For graph, we select variables that are (1) no parents, (2) 2 parents, but the two parents have no precedents, and (3) 2 parents, and parents have other precedents.
This setting makes (2) identical for naive inference and MLE based on Bayesian Network Architecture, and (3) will present the difference.
Experimental details are in Appendix~\ref{experiment-details}.

\begin{figure}[h]
    \centering
    \begin{subfigure}
        \centering
        \includegraphics[width=0.31\linewidth]{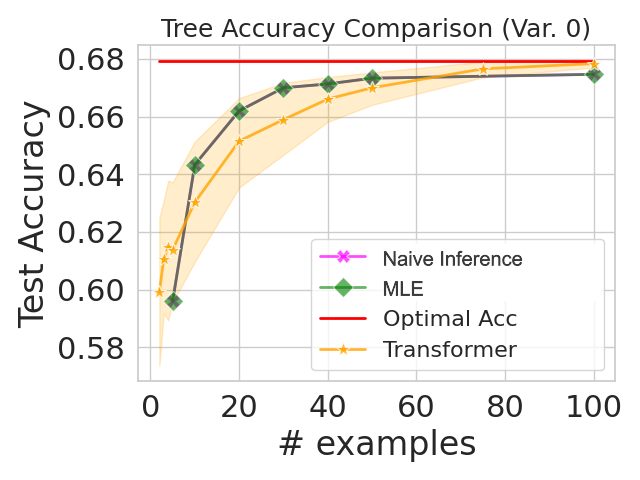}%
        \hfill\includegraphics[width=0.31\linewidth]{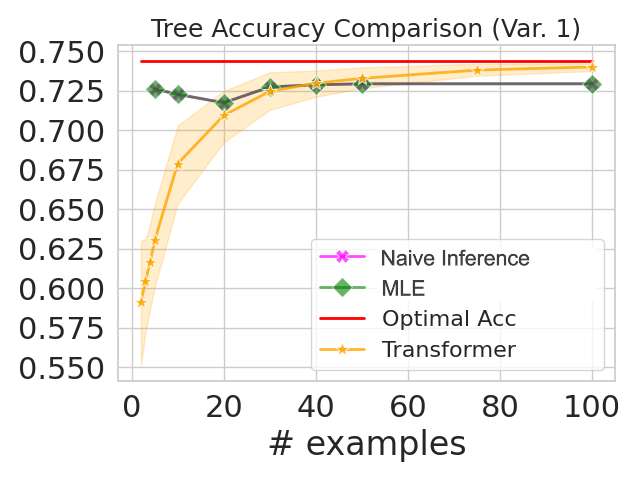}
        \hfill\includegraphics[width=0.31\linewidth]{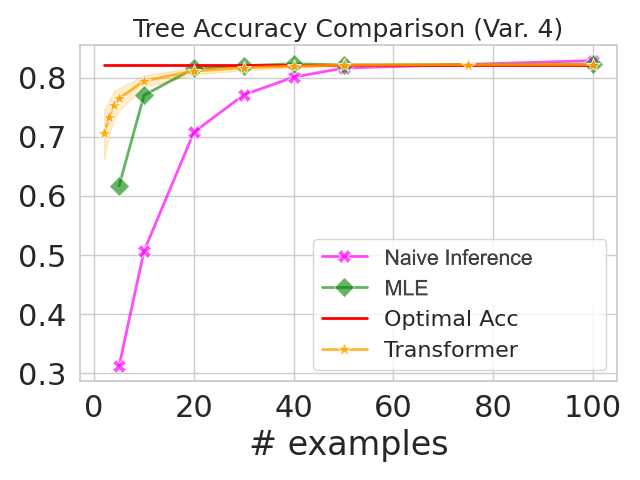}
        % \caption{Trees}
    \end{subfigure}
    \vspace{-0.5em}
    \begin{subfigure}
        \centering
        \includegraphics[width=0.31\linewidth]{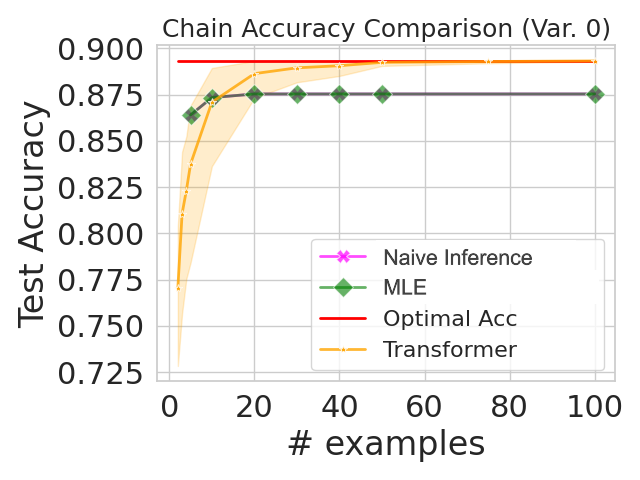}%
        \hfill\includegraphics[width=0.31\linewidth]{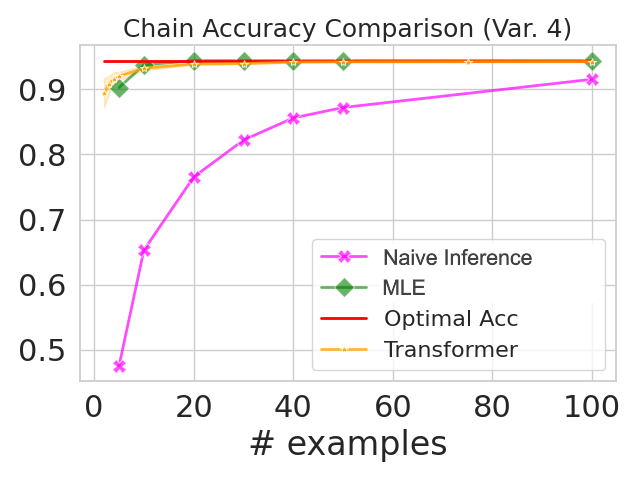}
        \hfill\includegraphics[width=0.31\linewidth]{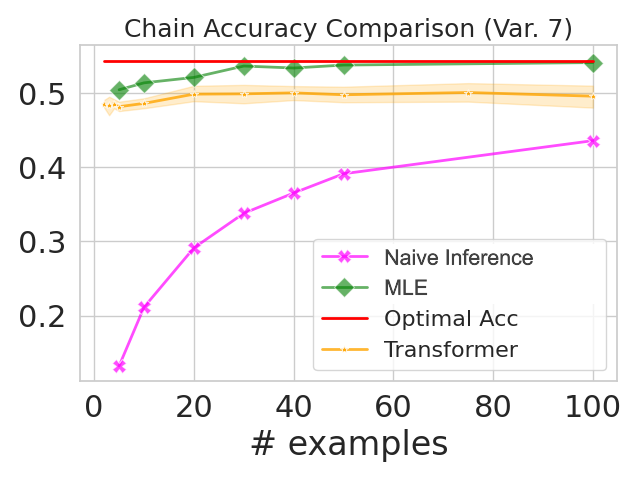}
    \end{subfigure}
    \vspace{-0.5em}
    \begin{subfigure}
        \centering
        \includegraphics[width=0.31\linewidth]{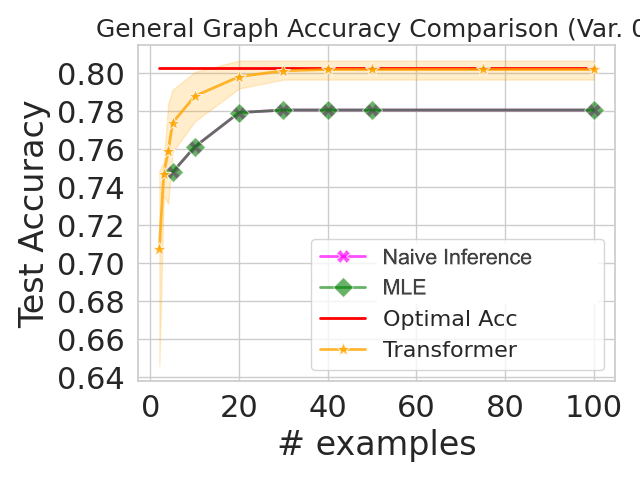}%
        \hfill\includegraphics[width=0.31\linewidth]{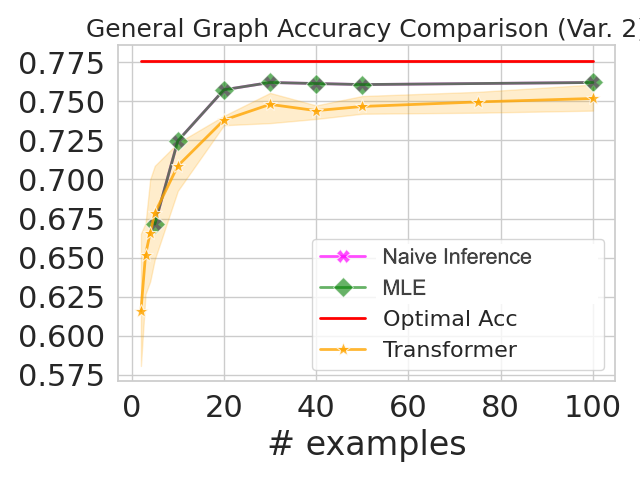}
        \hfill\includegraphics[width=0.31\linewidth]{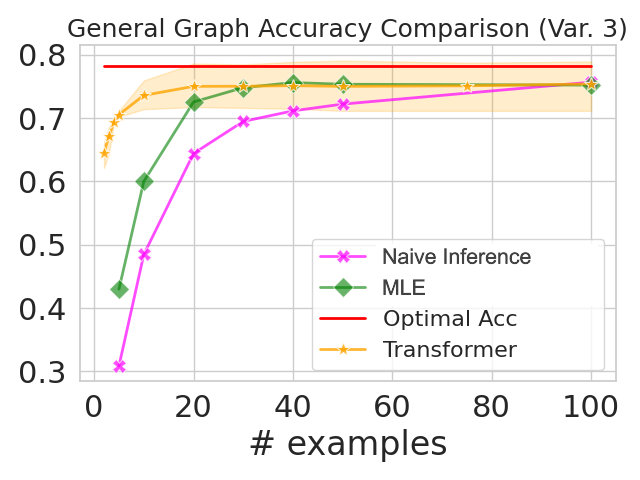}
    \end{subfigure}
    \vspace{-1em}
    \caption{
    \textbf{Top to Bottom: The Accuracy Comparison on Tree, Chain and General Graphs.}
    % We select 3 variables in each graph.
    We observe transformers present similar performance with MLE (short for MLE based on true network) and show better sample efficiency comparing to naive inference, indicating transformers are capable to model relationships between variables according to graph structure.
    % Moreover, MLE based on Bayesian network architecture and naive inference fail to generate prediction when the test token was never observed in the provided examples.
    % However, transformers are able to generate predictions based on its learned prior, showing its superior performance under few examples.
    }\label{fig:main-result}
\end{figure}

\textbf{Inference Results.}
The test accuracy results are in Figure~\ref{fig:main-result}.
Note that naive inference is able to model the first few variables in the selected graphs well as shown in the first column of Figure~\ref{fig:main-result}.
For general graph, variables 2, 3 both have 2 parents.
However, modeling variable 0, 1 is identical for naive inference and MLE based on true network, and is \textbf{NOT} for variable 3.
For tree, modeling root and variable 1 is identical for naive inference and MLE based on true network.
For chain, modeling variable 0, 1 is identical for naive inference and MLE based on Bayesian network architecture.
However, as the order of the variable goes further, transformers outperforms naive inference on both sample efficiency and test accuracy.
Moreover, MLE based on Bayesian Network Architecture and naive inference fail to generate prediction when the test token was never observed in the provided examples.
However, transformers are able to generate predictions based on its learned prior, showing its superior performance under few examples.
This indicates transformers are able to utilize the graph structure to generate prediction instead of treating all variables as independent observations.
Notably, while our transformers are only trained on samples with $N_{\text{train}}=100$, they are able to generalized to different values of $N_{\text{test}}$, and their test accuracy approaches to MLE based on Bayesian Network Architecture when $N_{\text{test}}$ increases.
This again verifies the capability of transformers to learn MLE based on Bayesian Network Architecture and model graph structure well.
Another thing to highlight is that both naive inference and MLE based on Bayesian Network Architecture are not capable of handling unseen observations, leading to assigning 0 probability on every outcome under this case.
However, transformers are able to utilize its learned prior from training data to perform prediction.
This explains why transformers outperforms the MLE based on Bayesian network baseline sometimes when $N_{\text{test}}$ is small.

\begin{figure}[th]
    \centering
    \includegraphics[width=\linewidth]{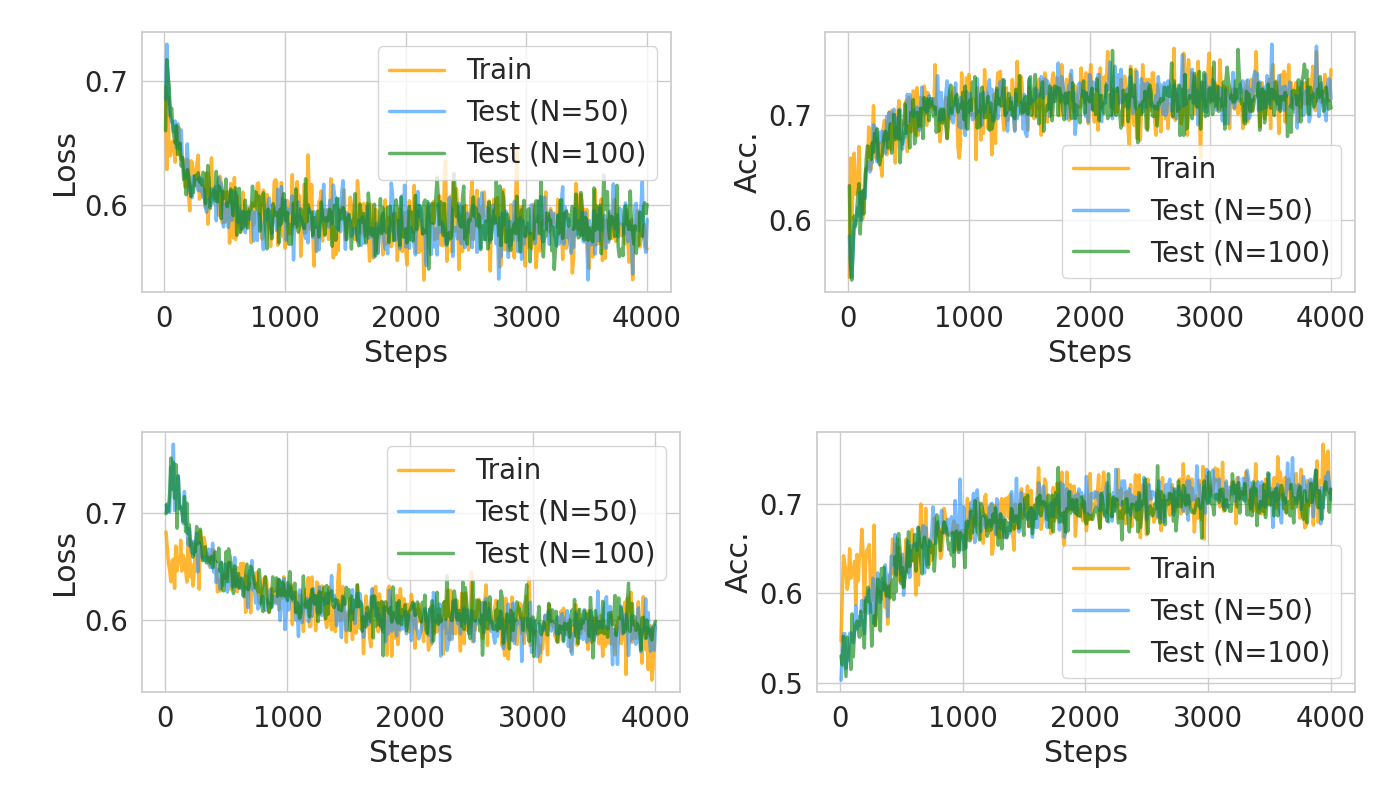}
    \caption{\textbf{Top: Convergence result on general graph.
    Bottom: Convergence result on tree.
    }
    We track the convergence result of transformers trained on general graph.
    Overall, we observe a decreasing trend of loss and increasing trend of accuracy on both training and test data.
    We also see that transformers are able to generalize well on the $N_{\text{test}}=50$ case even when its trained with $N_{\text{test}}=100$.
    % This shows transformers are capable of doing MLE based on Bayesian network architecture instead of learning inductive bias of data.
    % Note that during the beginning of training, transformers perform better on trainset than testset.
    % This is due to the curriculum design that the model is initially exposed to only 2 variables, while evaluated on the whole graph.
    }
    \label{fig:graphs-convergence}
\end{figure}

\textbf{Convergence Results.}
We now discuss the convergence result of transformers training on general graph and tree in Figure~\ref{fig:graphs-convergence}.
We show the loss and accuracy curve on training and test dataset throughout the optimization process.
We also observe the generalization performance on $N$ of transformers.
Specifically, we train models on $N=100$, and evaluate them on both $N=100$ and $N=50$ cases.
We observe that the loss curve presents a decreasing trend, and the accuracy is able to reach near optimal ($\sim 0.75$)\footnote{This is a rouge estimation based on our design of probability distributions of training data.} .
Note that all training and test samples are sampled from Bayesian networks.
Therefore, the optimal loss and accuracy are not $0$ and $1$, respectively.
Further, the generalization performance matches the results in Figure~\ref{fig:main-result}, as we see transformers are capable of performing MLE based on Bayesian network architecture under different $N_{\text{test}}$.

\subsection{Generalization Analysis}
Here we analyze when transformers trained on a fixed number of examples, which we denote $N_{\text{train}}$, whether it can generalize to different number of $N_{\text{test}}$.
We evaluate 2 cases:
(1) $N_{\text{train}} \gg N_{\text{test}}$,
(2) $N_{\text{train}} \ll N_{\text{test}}$.
Note that in our construction, $N$ does not affect transformers ability to perform MLE based on Bayesian network architecture.
However, during training, small $N_{\text{train}}$ can produce large noise, whereas larger $N_{\text{train}}$, while being more stable, can be easily modeled by naive inference.
This raises a doubt that whether transformers trained under larger $N_{\text{train}}$ learn naive inference or MLE based on Bayesian network architecture.
Therefore, we train transformers with $N_{\text{train}} \in \{5, 10, 200, 400\}$, and evaluate them with different $N_{\text{test}}$.
We also report the loss and accuracy curve during training and use $N_{\text{test}} \in \{20, 50\}$ as testset. 
The choice of these numbers is based on the fact that these numbers are effective to show the gap between MLE based on Bayesian network architecture and naive inference.
We present the results on general graph in the main paper, the generalization analysis on tree can be found in Appendix~\ref{generalization-tree}.

\paragraph{Results.}
The convergence and inference results are in Figure~\ref{fig:graphs-inference-gen}, Figure~\ref{fig:graphs-convergence-gen-left} and Figure~\ref{fig:graphs-convergence-gen-right}, respectively.
For the convergence result, we observe that models trained on large $N_{\text{train}}$ is able to generalize well on both $N_{\text{test}} = 20, 50$ (accuracy above 0.7).
However, for models trained under small $N_{\text{train}}$, they do not converge well and also do not generalize well on testset (accuracy below 0.7).
For the inference result, we see that models trained on large $N_{\text{train}}$ is capable of performing MLE based on Bayesian network architecture.
But models trained under small $N_{\text{train}}$ struggle to utilize the network structure to predict.
A potential reason is smaller $N_{\text{train}}$ is not sufficient to approximate the ground truth probability distribution well.
Also, while models trained on $N_{\text{train}}=400$ is almost equivalent to learning on independent variables, modeling are still able to learn the network structure, potentially show the positive effect of curriculum.
The result indicates a sufficient large $N_{\text{train}}$ is critical for transformers to learn MLE based on Bayesian network architecture in-context, providing practical insights on tasks in real-world scenarios.

\begin{figure}[t]
    \centering
    \includegraphics[width=\linewidth]{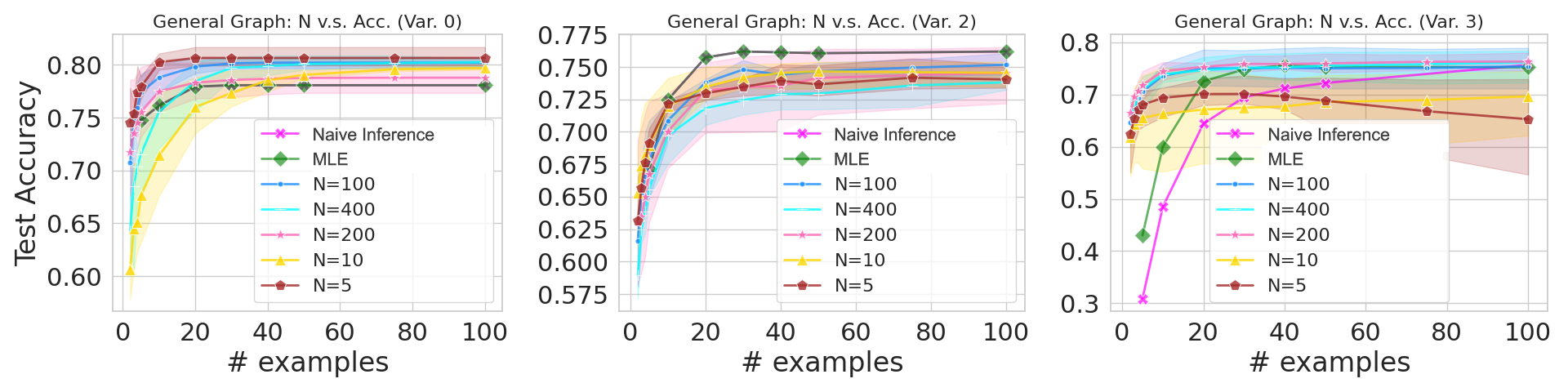}
    \vspace{-2em}
    \caption{\textbf{Left to right:
    Transformer's performance on general graph variable 0, 2, 3.
    }
    For variable 0, 2, all models are able to model the variable distributions well.
    Interestingly, for variable 3, transformers trained under $N_{\text{train}} = [5, 10]$ are not capable of predicting it well.
    Moreover, its performance is even worse than naive inference for large $N_{\text{test}}$.
    The result indicates that a sufficient size of $N_{\text{train}}$ is necessary for transformers to learn the network structure.
    }
    \label{fig:graphs-inference-gen}
\end{figure}

\begin{figure}[t]
    \centering
    \includegraphics[width=\linewidth]{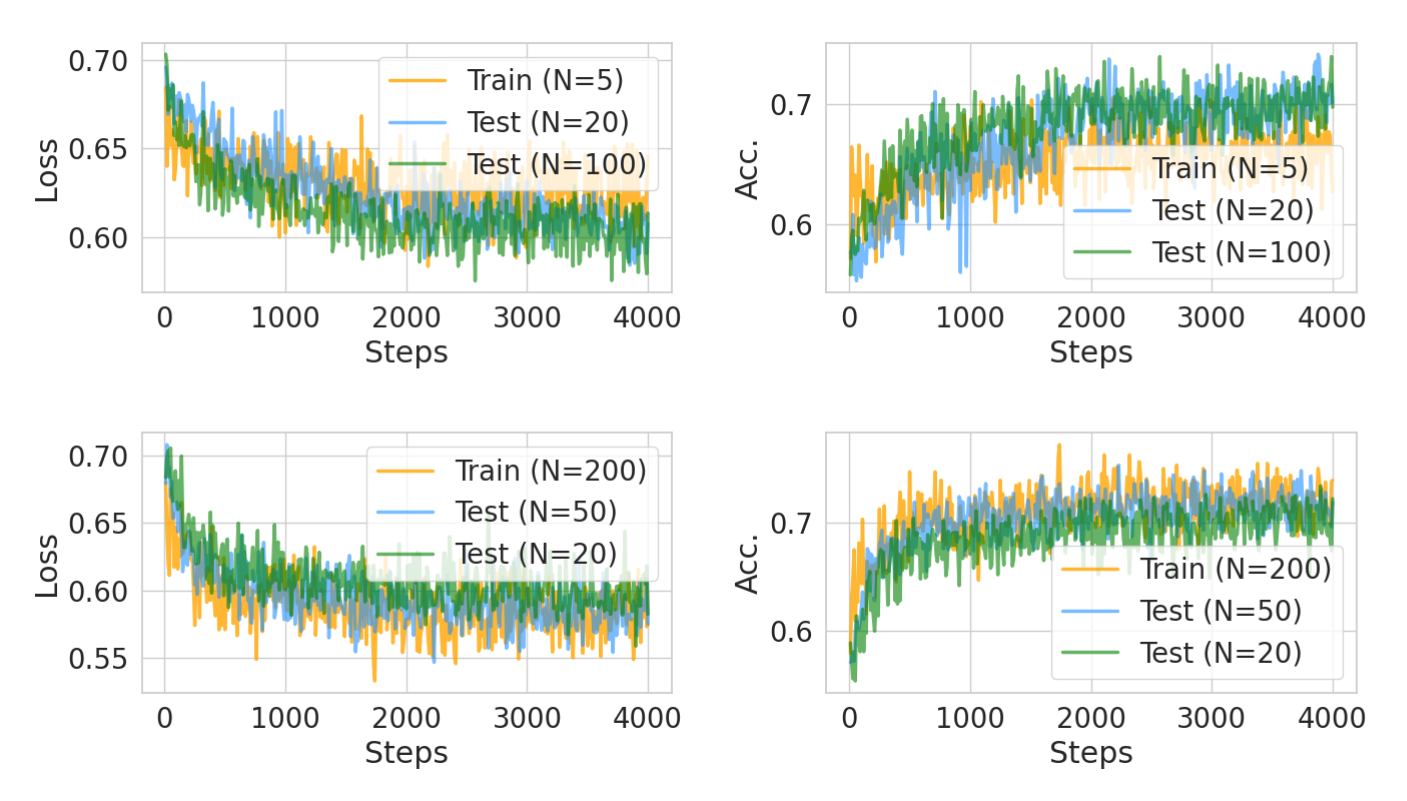}
    \vspace{-2em}
    \caption{\textbf{Convergence result on general graph for $N_{\text{train}} \in \{ 5, 10, 200, 400\}$.
    }
    Here we observe an obvious contrast between models trained on large and small $N_{\text{train}}$.
    For smaller $N_{\text{train}}$, model performance on training dataset is lower than testset.
    For larger $N_{\text{train}}$, we observe the opposite.
    We believe this is due to the fact that smaller $N_{\text{train}}$ does not provide sufficient sample size to recover the probability distribution well.
    }
    \label{fig:graphs-convergence-gen-left}
\end{figure}

\begin{figure}[t]
    \centering
    \includegraphics[width=\linewidth]{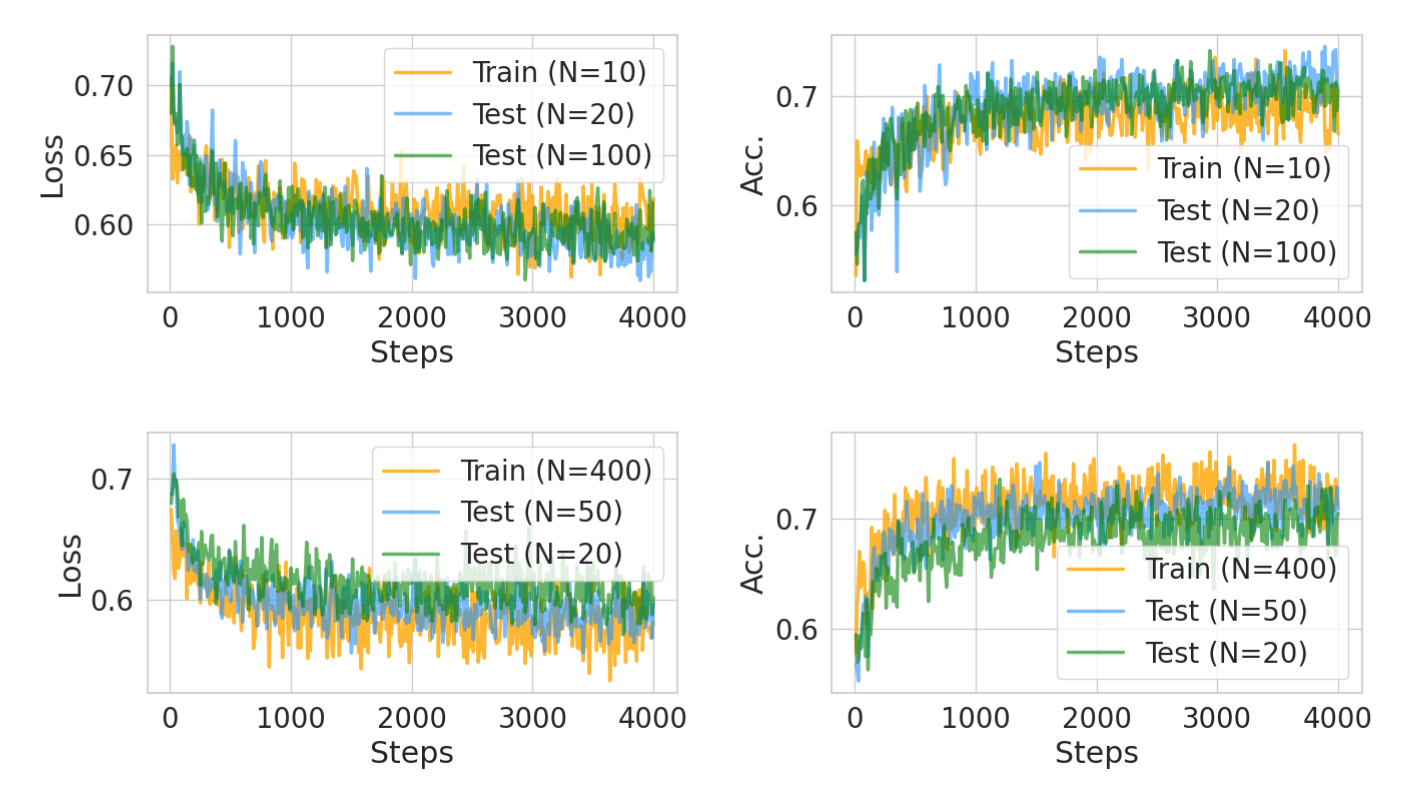}
    \vspace{-2em}
    \caption{\textbf{
     Convergence result on tree for $N_{\text{train}} \in \{ 5, 10, 200, 400\}$.
    }
    Similar to the convergence results on general graph, we also observe contrast between models trained on large and small $N_{\text{train}}$.
    Larger $N_{\text{train}}$ leads to better generalization while smaller $N_{\text{train}}$ leads to performance degradation
    }
    \label{fig:graphs-convergence-gen-right}
\end{figure}

\subsection{Real World Dataset}
Here we conduct experiments on the American Community Survey Income (ACSIncome) dataset from US Census.
The task is the predict whether the individual has an annual income over 50K U.S. dollars.

\paragraph{ACSIncome.}
The task is a binary classification problem with categorical features.
The ACSIncome dataset encompasses five years of data from approximately 3.5 million U.S. households, including information on citizenship, education, employment, marital status, and other attributes. 
The objective of this study is to predict whether an individual's annual income exceeds \$50,000. 
We utilize the version curated by \cite{ding2021retiring}, which excludes individuals younger than 16 years of age and those who worked fewer than 1 hour per week in the previous year. 
The income threshold of \$50,000 is consistent with that used in the UCI Adult Dataset \cite{adult_2}.

\paragraph{Distribution Shift in the Dataset.}
According to the analysis in \cite{liu2024need}, both the ACSIncome and ACSPublicCoverage datasets exhibit significant distributional shifts across different U.S. states and years of data collection, indicating strong heterogeneity in the conditional distributions across states and years. According to our analysis, transformers have the strength to capture dependency relationships among variables (i.e., the graph structure of the Bayesian network) by utilizing all available data, while estimating conditional probabilities based on the context. Therefore, we expect that transformers may offer potential benefits for these tasks.

% a potential benefit of using transformers for these tasks, as 

% Bayesian Inference for these tasks, where parameter heterogeneity occurs due to the difference of states and years.

\paragraph{Setup.}
We partitioned the data by state, designating one state (CA) as the test set and the rest as the training set.
We consider each state and year combination is a context, i.e., $(\text{CA}, 2014)$ and $(\text{CA}, 2015)$ and $(\text{MA}, 2014)$ are three different contexts, meaning they share the same Bayesian network structure, but has different parameters.
Our training data contains $245$ contexts ($49 \times 5$), which contains $5$ years of data of $49$ states in the US.
Our testset contains $5$ years of data of the state of California (CA), and we evaluate model's performance on these $5$ years separately since we assume they are $5$ different contexts. 
There are 10 variables in the network/feature, with different dimensions.
To simply the scope of the experiments, we merge some categories (within a variable) together as described in Appendix~\ref{experiment-details:real-data}.
For more training details, please also refer to Appendix~\ref{appendix:exp-details}.

\paragraph{Baselines.}
We compare Transformers to a 2-layered FeedForward ReLU Network (FFN).
Similar to our settings for the synthetic dataset, we vary $N_{\text{test}}$ for transformers.
For FFN, we use $N_{\text{test}}$ as the size of their training data, and train FFN with it.
Note that in our synthetic settings, baselines like naive inference and MLE are also only exposed to the ICL examples.
Therefore, we conduct our experiments on FFN with the same approach.
The 2-layer FFN has hidden dimension of $(50, 100)$.
We repeat both baselines for 20 runs and plot the average and standard deviation of their test accuracy.
We use $N_{\text{test}} = [5, 100, 200, 
 300, 400, 500, 600]$.

\paragraph{Results.}
The results are in Figure~\ref{fig:real-world-exp}.
We observe that when learning to perform MLE on the Bayesian network, transformers are able to improve its performance with larger ICL example sizes.
Note that the weights of transformer remain unchanged, indicating that the provided ICL examples provide useful information about the context distribution.
The results indicate that our theoretical insights also provide practical guidance to real world applications.

\begin{figure}[h]
    \centering
    \begin{subfigure}
        \centering
        \includegraphics[width=\linewidth]{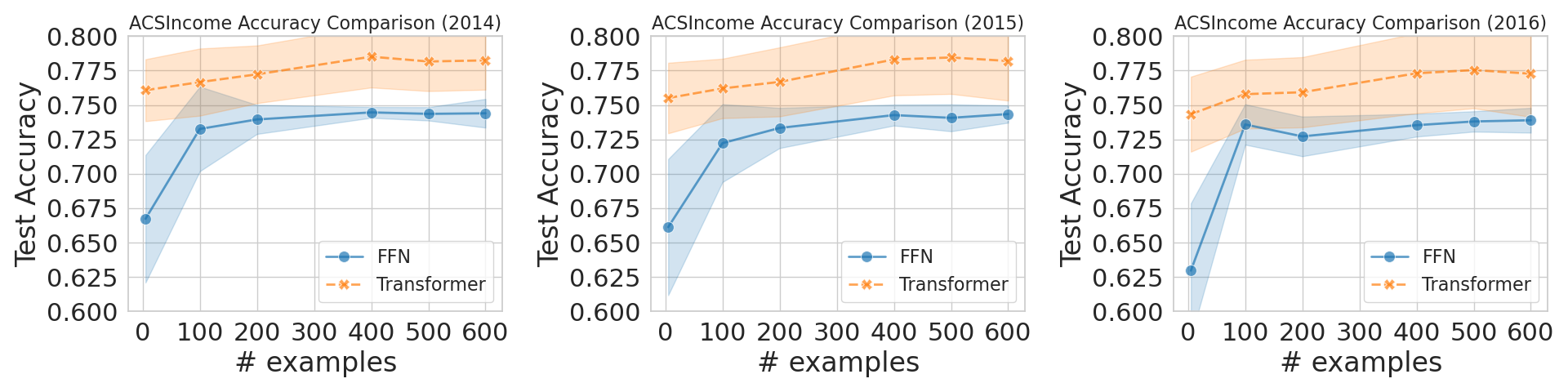}%
        % \caption{Trees}
    \end{subfigure}
    \vspace{-0.5em}
    \begin{subfigure}
        \centering
        \includegraphics[width=0.65\linewidth]{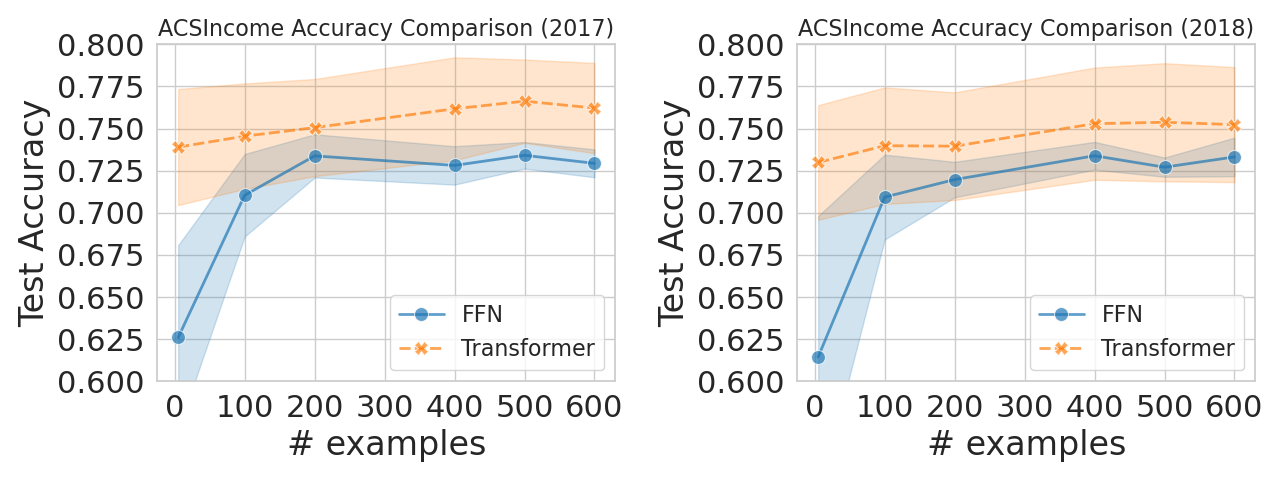}%
        % \hfill\includegraphics[width=0.33\linewidth,
        % trim={0 0 0 0}, clip]{2018.png}
    \end{subfigure}
    \vspace{-0.5em}
    \caption{
    \textbf{Top: The Accuracy comparison for Year 2014, 2015, and 2016.
    Bottom:
    Accuracy comparison for Year 2017 and 2018.
    }
    We are able to see that with the number of examples increases, transformer is able to perform better while weights being unchanged.
    This implies that more ICL examples provides useful information for the transformer to predict.
    }\label{fig:real-world-exp}
\end{figure}

\section{Proof sketch}
In this section, we give a proof sketch of Theorem~\ref{thm:main}. Te proof is based on relatively intuitive constructions of the two transformer layers. The result for the first transformer layer is summarized into the following lemma.
\begin{lemma}\label{lemma:first_layer}
    For any Bayesian network $\cB$ with maximum in-degree $D$, there exists a one-layer transformer $\mathrm{TF}_{\btheta^{(1)}}( \cdot )$ with parameter matrices satisfying 
    $\|\Vb^{(1)} \|_2, \|\Kb^{(1)} \|_2, \|\Qb^{(1)} \|_2, \|\Wb_2^{(1)} \|_2 \leq 1$ and $ \|\Wb_1^{(1)} \|_2\leq 2\sqrt{D+1}$,
    such that for any the variable-of-interest index $m_0$, it holds that % and the corresponding $\pb$ and $\pb_q$.
    \begin{align*}
        \mathrm{TF}_{\btheta^{(1)}}( \Xb ) = \tilde{\Xb} := \begin{bmatrix}
        \tilde\xb_{11} & \tilde\xb_{12} & \cdots & \tilde\xb_{1N} & \tilde\xb_{1q}  \\
        \tilde\xb_{21} & \tilde\xb_{22} & \cdots & \tilde\xb_{2N} & \tilde\xb_{2q}\\
        \vdots & \vdots & & \vdots & \vdots\\
        \tilde\xb_{M1} & \tilde\xb_{M2} & \cdots & \tilde\xb_{MN}  & \tilde\xb_{Mq}\\
        \pb & \pb  & \cdots & \pb &\mathbf{p}_{q}
    \end{bmatrix},
    \end{align*}
    where 
    \begin{align*}
        \tilde\xb_{mi} = \left\{ \begin{aligned}
&\xb_{mi}, &&\text{if } m \in \{ m_0 \}\cup \cP(m_0);\\
&\mathbf{0}, &&\text{otherwise}.
\end{aligned}
\right.,~~\tilde\xb_{mq} = \left\{ \begin{aligned}
&\xb_{mq}, &&\text{if } m \in \{ m_0 \}\cup \cP(m_0);\\
&\mathbf{0}, &&\text{otherwise}.
\end{aligned}
\right.
\end{align*}
for all $i\in[N]$.
%     \begin{align*}
%         \tilde\xb_{mq} = \left\{ \begin{aligned}
% &\xb_{mq}, &&\text{if } m \in \{ m_0 \}\cup \cP(m_0);\\
% &\mathbf{0}, &&\text{otherwise}.
% \end{aligned}
% \right.
% \end{align*}
    % $\tilde\xb_{mq} = \xb_{mq}$ if $m= m_0$ or $X_m$ is a parent of $X_{m_0}$, and $\tilde\xb_{mq} = \xb_{mq}$ otherwise.
\end{lemma}
Lemma~\ref{lemma:first_layer} above shows that, there exists a transformer layer with bounded weight matrices that can serves as a ``parents selector'' -- for any $m_0\in [M]$, as long as the ``positional embeddings'' $\pb$ and $\pb_q$ are defined accordingly, the output of the transformer layer will retain only the values of the observed variables that are direct parents of the $m_0$-th variable. This operation, which trims all non-essential observation values, effectively prepares for the in-context estimation of the conditional probabilities in the second layer.

The following lemma gives the result for the second transformer layer, which takes the output $\tilde\Xb$ of the first layer given in Lemma~\ref{lemma:first_layer} as input. 
\begin{lemma}\label{lemma:second_layer}
For any $\epsilon > 0$ and any Bayesian network $\cB$ with maximum in-degree $D$, there exists a one-layer transformer $\mathrm{TF}_{\btheta^{(2)}}( \cdot )$ with parameter matrices satisfying 
\begin{align*}
    \|\Vb^{(2)} \|_2, \|\Wb_1^{(2)} \|_2, \|\Wb_2^{(2)} \|_2  \leq 1, ~~\|\Kb^{(2)} \|_2, \|\Qb^{(2)} \|_2\leq 3\log(MdN/\epsilon),
\end{align*}
    % $\|\Vb^{(2)} \|_2, \|\Wb_1^{(2)} \|_2, \|\Wb_2^{(2)} \|_2  \leq 1$ and $ \|\Kb^{(2)} \|_2, \|\Qb^{(2)} \|_2\leq 3\log(MdN/\epsilon)$,
    such that for any index of the variable-of-interest $m_0$ and the corresponding $\tilde\Xb$ defined in Lemma~\ref{lemma:first_layer}, it holds that
\begin{align*}
    \mathrm{Read}\big[ \mathrm{TF}_{\btheta^{(2)}}( \tilde\Xb ) \big] = \hat\xb_q +  \sbb,
\end{align*}
where $\hat\xb_q = [\mathbf{0}_{(m_0-1)d}^\top,-\hat\xb_{m_0 q}^\top, \mathbf{0}_{(2M-m_0)d}^\top, \mathbf{1}_{d} ^\top]^\top$ with $\hat\xb_{m_0 q} = \pb^{\mathrm{MLE}}_{m_0}$, 
% \begin{align*}
%     \hat\xb_{m_0 q} = \pb^{\mathrm{MLE}}_{m_0} % \sum_{i\in[N]} \xb_{m_0i} \frac{ \ind[  \xb_{mi} = \xb_{mq} \text{ for all }m\in \cP(m_0) ] }{| \{ i\in [N]: \xb_{mi} = \xb_{mq} \text{ for all }m\in \cP(m_0) \}|},
% \end{align*}
and $\sbb\in \RR^{(2M+1)d}$ satisfies that $ \| \sbb \|_\infty \leq \epsilon / [(2M+1)d] $, and $\sum_{i=(m_0-1) d + 1)}^{m_0d} \sbb_i = 0$.
% \begin{align*}
%     \mathbf{TF_{\btheta^{(2)}}}(\tilde\Xb) = \begin{bmatrix}
%         \hat\xb_{11} & \hat\xb_{12} & \cdots & \hat\xb_{1N} & \hat\xb_{1q}  \\
%         \hat\xb_{21} & \hat\xb_{22} & \cdots & \hat\xb_{2N} & \hat\xb_{2q}\\
%         \vdots & \vdots & & \vdots & \vdots\\
%         \hat\xb_{M1} & \hat\xb_{M2} & \cdots & \hat\xb_{MN}  & \hat\xb_{Mq}\\
%         \hat\pb & \hat\pb  & \cdots & \hat\pb & \hat\pb_{q}
%     \end{bmatrix},
% \end{align*}
% where 
% \begin{align*}
%     \Bigg \| \hat\xb_{m_0q} - \sum_{i=1}^N \xb_{m_0i} \cdot \frac{ \ind[ \xb_{mi} = \xb_{mq} \text{ for all }m\in \cP(m_0)] }{ | \{ i\in [N]: \xb_{mi} = \xb_{mq} \text{ for all }m\in \cP(m_0) \} | } \Bigg\|_2\leq \epsilon.
% \end{align*}
% Moreover, 
% $\| \hat\xb_{mi} \|_2\leq \epsilon$ for all $m\in [M]$ and $i\in[N]$, $\| \hat\xb_{mq} \|_2\leq \epsilon$ for all $m \neq m_0$. 
\end{lemma}
Lemma~\ref{lemma:second_layer} shows that, there exists a transformer layer which takes the output of $\tilde\Xb$ defined in Lemma~\ref{lemma:first_layer} as input, and outputs a matrix whose last column is directly related to the target optimal maximum likelihood estimation $\pb_{m_0}^{\mathrm{MLE}}$. 

Given the two lemmas above, the proof of Theorem~\ref{thm:main} is straightforward. The proof is as follows.

\begin{proof}[Proof of Theorem~\ref{thm:main}]
Let $\mathrm{TF}_{\btheta^{(1)}}( \cdot )$ and $\mathrm{TF}_{\btheta^{(2)}}( \cdot )$ be defined in Lemmas~\ref{lemma:first_layer} and \ref{lemma:second_layer} respectively. Then we directly have
\begin{align*}
    \mathrm{Read}\big[ \mathrm{TF}_{\btheta^{(2)}}(  \mathrm{TF}_{\btheta^{(1)}}( \Xb ) ) \big] = \hat\xb_q +  \sbb,
\end{align*}
where $\hat\xb_q = [\mathbf{0}_{(m_0-1)d}^\top,-\hat\xb_{m_0 q}^\top, \mathbf{0}_{(2M-m_0)d}^\top, \mathbf{1}_{d} ^\top]^\top$ with $\hat\xb_{m_0 q} =  \pb^{\mathrm{MLE}}_{m_0}$, 
% \begin{align*}
%     \hat\xb_{m_0 q} = \hat\xb_{m_0 q} = \pb^{\mathrm{MLE}}_{m_0} \sum_{i\in[N]} \xb_{m_0i} \frac{ \ind[  \xb_{mi} = \xb_{mq} \text{ for all }m\in \cP(m_0) ] }{| \{ i\in [N]: \xb_{mi} = \xb_{mq} \text{ for all }m\in \cP(m_0) \}|},
% \end{align*}
and $ \| \sbb \|_\infty \leq \epsilon / [(2M+1)d] $, $\sum_{i=(m_0-1) d + 1)}^{m_0d} \sbb_i = 0$.
Therefore, setting $\Ab = [ \mathbf{0}_{d\times(m_0-1)d },-\Ib_{d\times d}, \mathbf{0}_{d\times (2M-m_0+1)d}]$, we obtain
\begin{align*}
    \Ab \mathrm{Read}\big[ \mathrm{TF}_{\btheta^{(2)}}(  \mathrm{TF}_{\btheta^{(1)}}( \Xb ) ) \big] = \Ab \hat\xb_q + \Ab \sbb =  \pb^{\mathrm{MLE}}_{m_0} + \Ab \sbb.
\end{align*}
By definition, it is clear that $\| \Ab \sbb\|_\infty \leq \epsilon/d$ and $\sum_{i=1}^d [\Ab \sbb]_i = 0$. This implies that $\Ab \mathrm{Read}\big[ \mathrm{TF}_{\btheta^{(2)}}(  \mathrm{TF}_{\btheta^{(1)}}( \Xb ) ) \big]$ is a probability vector, and finishes the proof.  
\end{proof}

\section{Conclusion}
% Here we discuss the contribution, limitation, broader impacts and future works of this paper.
% \paragraph{Contributions.}
In this paper, we theoretically analyze transformer's capability to learn Bayesian networks in-context in an autoregressive fashion. We show that there exists a simple construction of transformer such that it can (1) estimate the conditional probabilities of the Bayesian network in-context, and (2) autoregressively generate a new sample based on the estimated conditional probabilities. This sheds light on the potential of transformers in probabilistic reasoning and their applicability in various machine learning tasks involving structured data. Empirically, we provide extensive experiments to show that transformers are indeed capable of learning Bayesian networks and generalize well on unseen probability distributions, verifying our theoretical construction. Our theoretical and experimental results provide not only greater insights on the understanding of transformers, but also practical guidance in training transformers on Bayesian networks.

There are still multiple important aspects which this paper does not cover. First of all, our current theoretical result only demonstrates the \textit{expressive power} of transformers in the sense that a good transformer model with reasonable weights exist. Our result does not directly cover whether such a transformer can indeed be obtained through training. Our experiments indicate a positive answer to this question, making theoretical demonstrations a promising future work direction. Moreover, our current analysis does not take the number of heads into consideration. As is discussed in \citet{nichani2024transformers}, multi-head attention may play an important role when learning Bayesian networks with complicated network structures. Studying the impact of multi-head attention is another important future work direction. 
% In this paper, we theoretically analyze transformer's capability to learn Bayesian networks in-context in an autoregressive fashion.
% We show that there exists a simple construction of transformer such that it can (1) estimate the conditional probabilities of the Bayesian network in-context, and 
% (2) autoregressively generate a new sample based on the estimated conditional probabilities.
% This sheds lights on \textbf{XXX} and \textbf{YYY}.
% Empirically, we provide extensive experiments to show that transformers are indeed capable of learning Bayesian networks and generalize well on unseen probability distributions, verifying our theoretical construction.
% Our theoretical and experimental results provide not only greater insights on the understanding of transformers, but also provide practical guidance in training transformers on Bayesian networks.

% \paragraph{Limitations.}
% Theoretically, our limitation includes AAA, BBB.

% \paragraph{Future Works.}
% This paper mainly discuss the setting of learning Bayesian networks with transformers, a potential direction is to bridge our discovery with some real-world scenarios such as language, vision modeling.

% \paragraph{Broader Impact.}

\bibliography{reference}
\bibliographystyle{agsm}

\newpage

\appendix

% \stopcontents
% \startcontents[sections]
% \printcontents[sections]{ }{1}{}

\begin{center}
{\large\bf SUPPLEMENTARY MATERIAL}
\end{center}

{
\setlength{\parskip}{-0em}
\startcontents[sections]
\printcontents[sections]{ }{1}{}
}

\section{Proofs}\label{appendix:proofs}
In this section, we give the proofs of Lemmas~\ref{lemma:first_layer} nad \ref{lemma:second_layer}. 
\subsection{Proof of Lemma~\ref{lemma:first_layer}}
The proof of Lemma~\ref{lemma:first_layer} is given as follows. 
\begin{proof}[Proof of Lemma~\ref{lemma:first_layer}]
    Let $\Vb^{(1)} = \mathbf{0}_{(2M+1)d\times (2M+1)d}$, $\Kb^{(1)} = \Qb^{(1)} = \mathbf{0}_{Md\times (2M+1)}$. Then clearly we have 
    \begin{align*}
        \mathrm{Attn}_{\Vb^{(1)}, \Kb^{(1)}, \Qb^{(1)}} (\Xb) = \Xb. 
    \end{align*}
    Moreover, let $\Ab = [ \Ab_{ij} ]_{M\times (M+2)} \in \RR^{Md\times (M+1)d}$ be a $M\times (M+1)$ block matrix where
    \begin{align*}
        \Ab_{ij} = \left\{ \begin{aligned}
&\Ib_{d\times d}, &&\text{if } j\leq M \text{ and } i\in \{j\} \cup \cP(j);
\\
&\mathbf{0}_{d\times d}, &&\text{otherwise}.
\end{aligned}
\right.
    \end{align*}
Then, let
    $\Wb_2^{(1)} = -\Ib_{(2M+1)d\times (2M+1)d}$, and 
    \begin{align*}
        \Wb_1^{(1)} = \begin{bmatrix}
            \Ib_{Md\times Md} & -2 \Ab \\
            \mathbf{0}_{(M+1)d\times Md} & 
            \mathbf{0}_{(M+1)d\times (M+1)d} 
        \end{bmatrix}.
    \end{align*}
    We note that the above definintion does not rely on any specific value of $m_0$. By definition, we can directly verify that
    \begin{align*}
        \Wb_1^{(1)} \Xb = \begin{bmatrix}
        \check\xb_{11} & \check\xb_{12} & \cdots & \check\xb_{1N} & \check\xb_{1q}  \\
        \check\xb_{21} & \check\xb_{22} & \cdots & \check\xb_{2N} & \check\xb_{2q}\\
        \vdots & \vdots & & \vdots & \vdots\\
        \check\xb_{M1} & \check\xb_{M2} & \cdots & \check\xb_{MN}  & \check\xb_{Mq}\\
        \mathbf{0}_{(M+1)d} & \mathbf{0}_{(M+1)d}  & \cdots & \mathbf{0}_{(M+1)d} & \mathbf{0}_{(M+1)d}
    \end{bmatrix},
    \end{align*}
    where $\check\xb_{mi} = \xb_{mi} - 2\mathbf{1}\cdot \ind[ m \in \{ m_0 \}\cup \cP(m_0) ]$. 
    Now since $\xb_{mi}$, $m\in[M]$, $i\in[N]$ are all one-hot vectors (and therefore have non-negative entries between zero and one), we see that the entries of $\check\xb_{mi}$ are strictly negative if and only if $m \in \{ m_0 \}\cup \cP(m_0)$. Therefore, by the definition of the ReLU activation function, we have
    \begin{align*}
        \sigma(\Wb_1^{(1)} \Xb ) = \begin{bmatrix}
        \overline\xb_{11} & \overline\xb_{12} & \cdots & \overline\xb_{1N} & \overline\xb_{1q}  \\
        \overline\xb_{21} & \overline\xb_{22} & \cdots & \overline\xb_{2N} & \overline\xb_{2q}\\
        \vdots & \vdots & & \vdots & \vdots\\
        \overline\xb_{M1} & \overline\xb_{M2} & \cdots & \overline\xb_{MN}  & \overline\xb_{Mq}\\
        \mathbf{0}_{(M+1)d} & \mathbf{0}_{(M+1)d}  & \cdots & \mathbf{0}_{(M+1)d} & \mathbf{0}_{(M+1)d}
    \end{bmatrix},
    \end{align*}
    %     \begin{bmatrix}
    %     \xb_{11} \cdot \ind\{1\notin \cP(m_0)\} & \cdots & \xb_{1N} \cdot \ind\{1\notin \cP(m_0)\}& \xb_{1q} \cdot \ind\{1\notin \cP(m_0)\} \\
    %     \xb_{21}\cdot \ind\{2\notin \cP(m_0)\} &  \cdots & \xb_{2N}\cdot \ind\{2\notin \cP(m_0)\} & \xb_{2q}\cdot \ind\{2\notin \cP(m_0)\}\\
    %     \vdots & & \vdots & \vdots\\
    %     \xb_{M1}\cdot \ind\{M\notin \cP(m_0)\} &  \cdots & \xb_{MN}\cdot \ind\{M\notin \cP(m_0)\}  & \xb_{Mq}\cdot \ind\{M\notin \cP(m_0)\}\\
    %     \mathbf{0}_{Md\times 1} &  \cdots & \mathbf{0}_{Md\times 1} &\mathbf{0}_{Md\times 1}
    % \end{bmatrix}.
    where  $\overline\xb_{mi} = \xb_{mi}\cdot \ind[ m \notin \{ m_0 \}\cup \cP(m_0) ]$. 
    Therefore, by $\Wb_2^{(1)} = -\Ib_{2Md\times 2Md}$, we have
    \begin{align*}
        \mathrm{TF}_{\btheta^{(1)} }( \Xb ) &= \mathrm{FF}_{\Wb_1^{(1)} ,\Wb_2^{(1)} } [ \mathrm{Attn}_{\Vb^{(1)} , \Kb^{(1)} , \Qb^{(1)} } (\Xb) ]
        = \mathrm{FF}_{\Wb_1^{(1)} ,\Wb_2^{(1)} } (\Xb)\\
        &=  \Xb + \Wb_2^{(1)}  \sigma( \Wb_1^{(1)}  \Xb ) = \Xb -   \sigma( \Wb_1^{(1)}  \Xb )\\
        &= \begin{bmatrix}
        \tilde\xb_{11} & \tilde\xb_{12} & \cdots & \tilde\xb_{1N} & \tilde\xb_{1q}  \\
        \tilde\xb_{21} & \tilde\xb_{22} & \cdots & \tilde\xb_{2N} & \tilde\xb_{2q}\\
        \vdots & \vdots & & \vdots & \vdots\\
        \tilde\xb_{M1} & \tilde\xb_{M2} & \cdots & \tilde\xb_{MN}  & \tilde\xb_{Mq}\\
        \pb & \pb  & \cdots & \pb &\mathbf{p}_{q}
    \end{bmatrix},
    \end{align*}
    % \begin{align*}
    %     \mathrm{TF}_{\btheta^{(1)}}( \Xb ) = \tilde{\Xb} := \begin{bmatrix}
    %     \tilde\xb_{11} & \tilde\xb_{12} & \cdots & \tilde\xb_{1N} & \tilde\xb_{1q}  \\
    %     \tilde\xb_{21} & \tilde\xb_{22} & \cdots & \tilde\xb_{2N} & \tilde\xb_{2q}\\
    %     \vdots & \vdots & & \vdots & \vdots\\
    %     \tilde\xb_{M1} & \tilde\xb_{M2} & \cdots & \tilde\xb_{MN}  & \tilde\xb_{Mq}\\
    %     \pb & \pb  & \cdots & \pb &\mathbf{p}_{q}
    % \end{bmatrix},
    % \end{align*}
    where 
    \begin{align*}
        \tilde\xb_{mi} = \left\{ \begin{aligned}
&\xb_{mi}, &&\text{if } m \in \{ m_0 \}\cup \cP(m_0);\\
&\mathbf{0}, &&\text{otherwise}.
\end{aligned}
\right.,~~\tilde\xb_{mq} = \left\{ \begin{aligned}
&\xb_{mq}, &&\text{if } m \in \{ m_0 \}\cup \cP(m_0);\\
&\mathbf{0}, &&\text{otherwise}.
\end{aligned}
\right.
\end{align*}
for all $i\in[N]$.
This finishes the proof.
\end{proof}

\subsection{Proof of Lemma~\ref{lemma:second_layer}}
We present the proof of Lemma~\ref{lemma:second_layer} as follows. 

\begin{proof}[Proof of Lemma~\ref{lemma:second_layer}]
    Clearly, by the definition of the $\mathrm{Read}(\cdot)$ function, only the last column of the output of $\mathrm{TF}_{\btheta^{(2)}}$ matters. Since the last column of the output of $\mathrm{TF}_{\btheta^{(2)}}$ only relies on the last column of $\mathrm{Attn}_{\Vb^{(2)}, \Kb^{(2)}, \Qb^{(2)}} (\tilde\Xb) $, we focus on the last column of $\mathrm{softmax}[ (\Kb \Xb)^\top  (\Qb \Xb)  ]$, which is $\mathrm{softmax}[ (\Kb \tilde\Xb)^\top  (\Qb \tilde\xb_q)  ]$, where $\tilde\xb_q = [\tilde\xb_{1q}^\top, \ldots, \tilde\xb_{Mq}^\top, \pb_q^\top ]^\top$. 
    Denote $c = \log( d/ \epsilon )$. Let $\Wb_1^{(2)} =  \Wb_2^{(2)} = \mathbf{0}_{(2M+1)d\times (2M+1)d}$,  $\Vb^{(2)} = -\Ib_{(2M+1)d\times (2M+1)d}$, and %$\Kb^{(2)} = [ \Ib_{Md\times Md} , \mathbf{0}_{Md \times (M+1)d } ]$
    %$\Qb^{(1)} = \mathbf{0}_{Md\times (2M+1)}$. Then clearly we have 
    \begin{align*}
        \Kb^{(2)} = \sqrt{c}\cdot \begin{bmatrix}
            \Ib_{Md\times Md} & \mathbf{0}_{Md\times Md} & \mathbf{0}_{Md\times d} \\
            \mathbf{0}_{d\times Md} & \mathbf{0}_{d\times Md} & \Ib_{d\times d}
        \end{bmatrix}, \quad \Qb^{(2)} = \sqrt{c}\cdot \begin{bmatrix}
            \Ib_{Md\times Md} & \mathbf{0}_{Md\times Md} & \mathbf{0}_{Md\times d} \\
            \mathbf{0}_{d\times Md} & \mathbf{0}_{d\times Md} & -\Ib_{d\times d}
        \end{bmatrix}.
    \end{align*}
Then we have 
\begin{align*}
    \Kb^{(2)} \tilde\Xb = \sqrt{c}\cdot \begin{bmatrix}
        \tilde\xb_{11} & \tilde\xb_{12} & \cdots & \tilde\xb_{1N} & \tilde\xb_{1q}  \\
        \tilde\xb_{21} & \tilde\xb_{22} & \cdots & \tilde\xb_{2N} & \tilde\xb_{2q}\\
        \vdots & \vdots & & \vdots & \vdots\\
        \tilde\xb_{M1} & \tilde\xb_{M2} & \cdots & \tilde\xb_{MN}  & \tilde\xb_{Mq}\\
        \mathbf{0}_{d} & \mathbf{0}_{d}  & \cdots & \mathbf{0}_{d} &  \mathbf{1}_{d}
    \end{bmatrix}, \quad \Qb^{(2)}\tilde\xb_q = \sqrt{c}\cdot \begin{bmatrix}
        \tilde\xb_{1q}  \\
        \tilde\xb_{2q}\\
        \vdots\\
        \tilde\xb_{Mq}\\
        - \mathbf{1}_{d}
    \end{bmatrix}.
\end{align*}
Recall the definition that
    \begin{align*}
        \tilde\xb_{mi} = \left\{ \begin{aligned}
&\xb_{mi}, &&\text{if } m \in \{ m_0 \}\cup \cP(m_0);\\
&\mathbf{0}, &&\text{otherwise}.
\end{aligned}
\right.,~~\tilde\xb_{mq} = \left\{ \begin{aligned}
&\xb_{mq}, &&\text{if } m \in \{ m_0 \}\cup \cP(m_0);\\
&\mathbf{0}, &&\text{otherwise}.
\end{aligned}
\right.
\end{align*}
for all $i\in[N]$. Therefore, for $i\in [N]$, we have
\begin{align*}
    (\Kb \tilde\xb_i)^\top  (\Qb \tilde\xb_q) &= c\cdot\sum_{m=1}^M \la \tilde\xb_{mi}, \tilde\xb_{mq} \ra\\
    &= c\cdot\sum_{m=1}^M \la \xb_{mi}, \xb_{mq} \ra \ind[ m \in \{ m_0 \}\cup \cP(m_0) ] \\
    &= c\cdot| \{ m\in \{ m_0 \}\cup \cP(m_0) :  \xb_{mi} = \xb_{mq}  \} |\\
    &= c\cdot| \{ m\in \cP(m_0) :  \xb_{mi} = \xb_{mq}  \} |,
\end{align*}
where the last equation is due to the fact that $\xb_{m_0q} = \mathbf{0}$, as it has not been sampled. 
Similarly, we also have
\begin{align*}
    (\Kb \tilde\xb_q)^\top  (\Qb \tilde\xb_q) = c\cdot\sum_{m=1}^M \la \tilde\xb_{mq}, \tilde\xb_{mq}\ra  -cd = c\cdot | \cP(m_0) | - cd.
\end{align*}

Now  denote $\cI(m_0) =  \{ i\in [N]: \xb_{mi} = \xb_{mq} \text{ for all }m\in \cP(m_0) \}$.  
Then for any $i \in \cI(m_0)$ (by assumption, this set is not empty), we have
\begin{align*}
    | \{ m\in \cP(m_0) :  \xb_{mi} = \xb_{mq}  \} | = |\cP(m_0)|.
\end{align*}
Therefore, 
for any $i \in\cI(m_0)$ and any $i'\notin \cI(m_0)$, we have
\begin{align*}
     (\Kb \tilde\xb_i)^\top  (\Qb \tilde\xb_q) - (\Kb \tilde\xb_{i'})^\top  (\Qb \tilde\xb_q) \geq c\cdot |\cP(m_0)| - c\cdot ( |\cP(m_0) | -1 ) = c.
\end{align*}
Moreover, 
\begin{align*}
    (\Kb \tilde\xb_i)^\top  (\Qb \tilde\xb_q) - (\Kb \tilde\xb_{q})^\top  (\Qb \tilde\xb_q)= c\cdot |\cP(m_0)| - c\cdot |\cP(m_0)| + cd = cd.
\end{align*}
Therefore, by $c = 3\log( MdN / \epsilon)$ we have
\begin{align*}
    \Bigg\|  \mathrm{softmax}[ (\Kb \tilde\Xb)^\top  (\Qb \tilde\xb_q) ] - \frac{1}{| \cI(m_0) |} \sum_{i\in \cI(m_0)} \eb_i \Bigg\|_\infty \leq\frac{\epsilon}{(2M+1)d}.
\end{align*}
Now by the choice that $\Vb^{(2)} = -\Ib_{(2M+1)d\times (2M+1)d}$, we have
\begin{align*}
    \mathrm{Read}\big[\mathrm{Attn}_{\Vb^{(2)}, \Kb^{(2)}, \Qb^{(2)}} (\tilde\Xb) \big] &= \tilde\xb_q + \Vb^{(2)} \tilde\Xb \mathrm{softmax}[ (\Kb^{(2)} \tilde\Xb)^\top  (\Qb^{(2)} \tilde\xb_q) ]\\
    & = \tilde\xb_q -  \frac{1}{| \cI(m_0) |} \sum_{i\in \cI(m_0)}\tilde\Xb \eb_i + \sbb
    % &= \hat\xb_q +  \sbb,
\end{align*}
where $\sbb \in \RR^{(2M+1)d}$ satisfies $ \| \sbb \|_\infty \leq \epsilon / [(2M+1)d] $ and $\sum_{i=(m_0-1) d + 1)}^{m_0d} \sbb_i = 0$. Now note that  (i) $\tilde\xb_{mi}$'s and $\tilde\xb_{mq}$'s are all zero except for $m\in \{m_0\}\cup \cP(m_0)$, (ii) for all $i\in \cI(m_0)$, and $m\in \cP(m_0)$, $\xb_{mi} = \xb_{mq}$. Therefore, on the right-hand side of the equation above, most of the terms are actually canceled when calculating the difference $\tilde\xb_q -  \frac{1}{| \cI(m_0) |} \sum_{i\in \cI(m_0)}\tilde\Xb \eb_i$. We have 
\begin{align*}
    \mathrm{Read}\big[\mathrm{Attn}_{\Vb^{(2)}, \Kb^{(2)}, \Qb^{(2)}} (\tilde\Xb) \big]
    &= \hat\xb_q +  \sbb,
\end{align*}
where 
$\hat\xb_q = [\mathbf{0}_{(m_0-1)d}^\top,-\hat\xb_{m_0 q}^\top, \mathbf{0}_{(2M-m_0)d}^\top, \mathbf{1}_{d} ^\top]^\top$, and
\begin{align*}
    \hat\xb_{m_0 q} &= \frac{1}{| \cI(m_0) |} \sum_{i\in \cI(m_0)}\tilde\xb_{m_0i} \\
    &= \frac{1}{| \cI(m_0) |} \sum_{i\in \cI(m_0)}\xb_{m_0i} \\
    &=  \sum_{i\in[N]} \xb_{m_0i} \frac{ \ind[  \xb_{mi} = \xb_{mq} \text{ for all }m\in \cP(m_0) ] }{| \{ i\in [N]: \xb_{mi} = \xb_{mq} \text{ for all }m\in \cP(m_0) \}|}.
\end{align*}
Now by $\Wb_1^{(2)} =  \mathbf{0}_{(2M+1)d\times (2M+1)d} $, $\Wb_2^{(2)} =\mathbf{0}_{(2M+1)d\times (2M+1)d}$, we have
\begin{align*}
    \mathrm{Read}\big[ \mathrm{TF}_{\btheta^{(2)}}( \tilde\Xb ) \big] &= \mathrm{Read}\big[ \mathrm{FF}_{\Wb_1^{(2)},\Wb_2^{(2)}} [ \mathrm{Attn}_{\Vb^{(2)}, \Kb^{(2)}, \Qb^{(2)}} (\tilde\Xb) ] \big]\\
    &= \mathrm{FF}_{\Wb_1^{(2)},\Wb_2^{(2)}}\big\{ \mathrm{Read}\big[\mathrm{Attn}_{\Vb^{(2)}, \Kb^{(2)}, \Qb^{(2)}} (\tilde\Xb) \big] \big\}\\
    &= \mathrm{FF}_{\Wb_1^{(2)},\Wb_2^{(2)}}( \hat\xb_q +  \sbb )\\
    &=  \hat\xb_q +  \sbb.
\end{align*}
This finishes the proof.
\end{proof}

% \newpage

\section{Additional Experiments}\label{appendix:additional-exp}
Here we conduct a hyperparameter analysis to see whether transformers are sensitive on certain hyperparameters.
It is also a more complete result of some experimental sections in main paper.
We analyze three hyperparameters:
\begin{itemize}
    \item Number of layers
    \item Number of attention heads
    \item $N_{\text{train}}$
\end{itemize}
We perform these analysis on general graph and select variable 0, 2, 3 to evaluate.
The reasoning behind this selection is to demonstrate 3 different properties of these variables.
For variable 0, it is a random variable without any parents, so modeling it is

\subsection{The Effect of Layers.}
Here we evaluate transformers with $\{1, 2, 6\}$ layers on general graph.
Overall, we want to observe whether the number of layers affect transformer's ability to learn MLE based on Bayesian network architecture.
The result is in Figure~\ref{fig:layer-impact}.

\begin{figure}[h]
    \centering
    \includegraphics[width=0.95\linewidth]{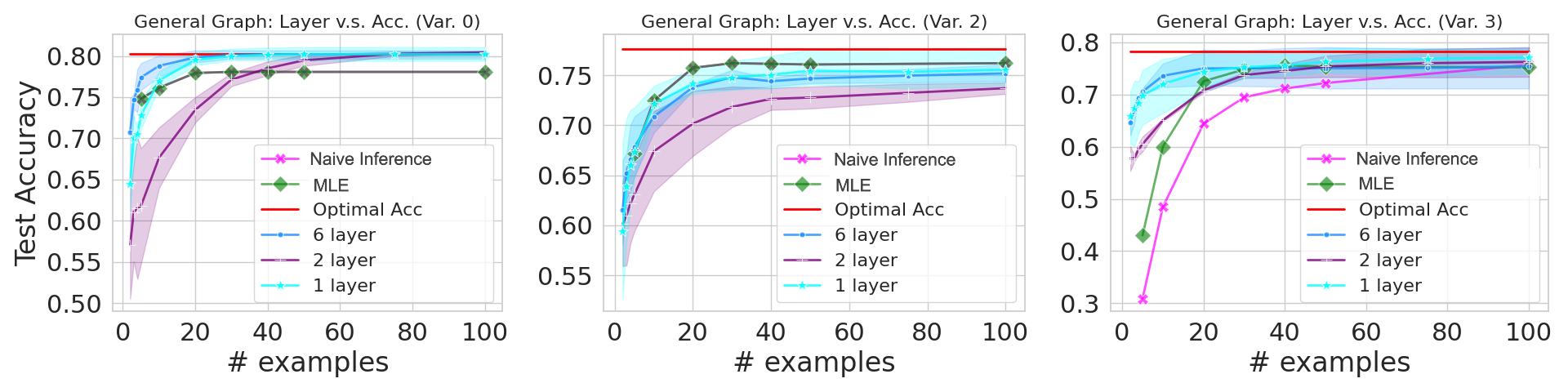}
    \vspace{-1em}
    \caption{\textbf{Evaluation of transformers with different layer on general graph.
    Left to right:
    variable 0, 2, 3.
    }
    We set the hidden dimension to $256$, number of heads to $8$ for all transformers.
    The result is the average taken over 5 runs.
    We observe that even the 2-layer transformer performs worse and presents larger variance, all transformers have similar behavior on this task.
    }
    \label{fig:layer-impact}
\end{figure}

\subsection{The Effect of Heads.}
Here we evaluate transformers with $\{1, 2, 4, 8\}$ attention heads on general graph.
Overall, we want to observe whether the number of attention heads affect transformer's ability to learn MLE based on Bayesian network architecture.
The result is in Figure~\ref{fig:head-impact}.
Empirically, we do not discover a significant impact of attention heads on models performance in our case study.
% As discussed in Section~\ref{subsection:optimal}, while we do not observe such an impact, it might due to the fact that the general graph structure is too simple to reflect such an architecture bottleneck.

\begin{figure}[h]
    \centering
    \includegraphics[width=0.95\linewidth]{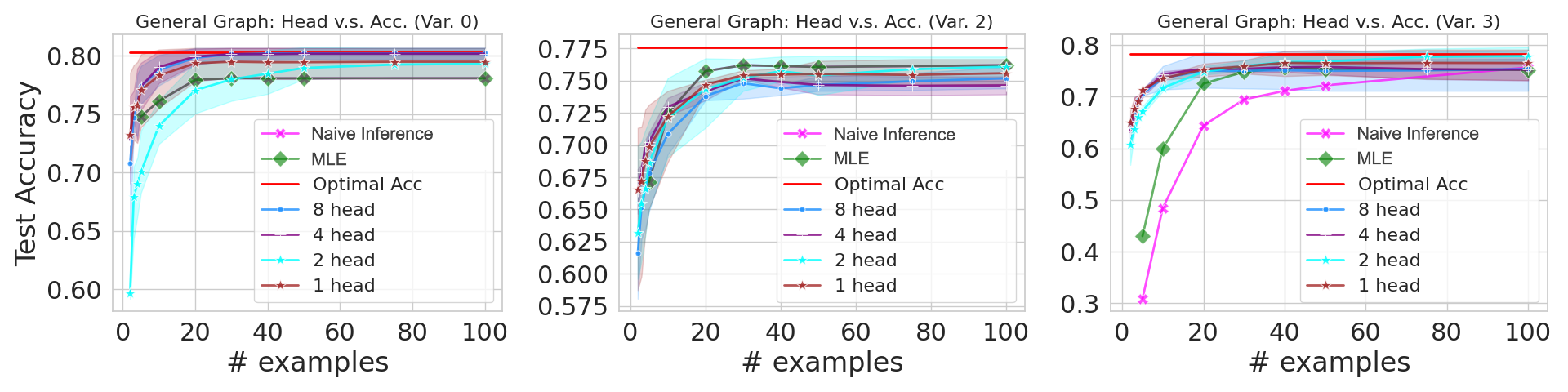}
    \caption{\textbf{Evaluation of transformers with different number of attention heads on general graph.
    Left to right:
    variable 0, 2, 3.
    }
    We set the hidden dimension to $256$, layer to $6$ for all transformers.
    The result is the average taken over 5 runs.
    Similar to the above subsection, we also do not observe significant performance degradation when reducing the number of heads.
    Especially for variable 3, which highly requires the network structure to inference prediction, transformer with 1-head still performs similar with its other variants.
    }
    \label{fig:head-impact}
\end{figure}
\subsection{The Effect of $N$ during Training.}
Here we evaluate transformers with values of $N_{\text{train}}$ on general graph and tree.
We aim to test models generalization capability and evaluate whether models require certain size of $N_{\text{train}}$ to learn MLE based on Bayesian network architecture in-context.
\paragraph{General Graph.}
The convergence and inference results are in Figure~\ref{fig:graphs-convergence-gen-app-left}, Figure~\ref{fig:graphs-convergence-gen-app-right} and Figure~\ref{fig:graphs-inference-gen-app}, respectively.
For the convergence result, we observe that models trained on large $N_{\text{train}}$ is able to generalize well on both $N_{\text{test}} = 20, 50$ (accuracy above 0.7).
However, for models trained under small $N_{\text{train}}$, they do not converge well and also do not generalize well on testset (accuracy below 0.7).
For the inference result, we see that models trained on large $N_{\text{train}}$ is capable of performing MLE based on Bayesian network architecture.
But models trained under small $N_{\text{train}}$ struggle to utilize the network structure to predict.
A potential reason is smaller $N_{\text{train}}$ is not sufficient to approximate the ground truth probability distribution well.
The result indicates that a sufficient large $N_{\text{train}}$ is critical for transformers to learn MLE based on Bayesian network architecture in-context, providing practical insights on real-world scenarios and downstream tasks.

\begin{figure}[t]
    \centering
    \includegraphics[width=\linewidth]{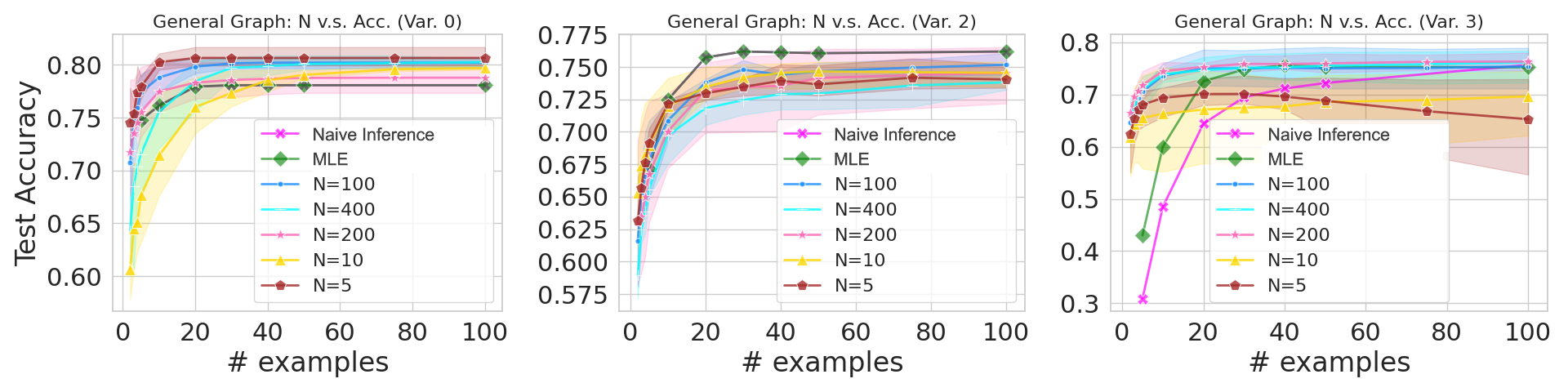}
    \vspace{-2em}
    \caption{\textbf{Left to right:
    Transformer's performance on general graph variable 0, 2, 3.
    }
    For variable 0, 2, all models are able to model the variable distributions well.
    Interestingly, for variable 3, transformers trained under $N_{\text{train}} = [5, 10]$ are not capable of predicting it well.
    Moreover, its performance even worse than naive inference for large $N_{\text{test}}$.
    The result indicates that a sufficient size of $N_{\text{train}}$ is necessary for transformers to learn the network structure.
    }
    \label{fig:graphs-inference-gen-app}
\end{figure}

\begin{figure}[t]
    \centering
    \includegraphics[width=\linewidth]{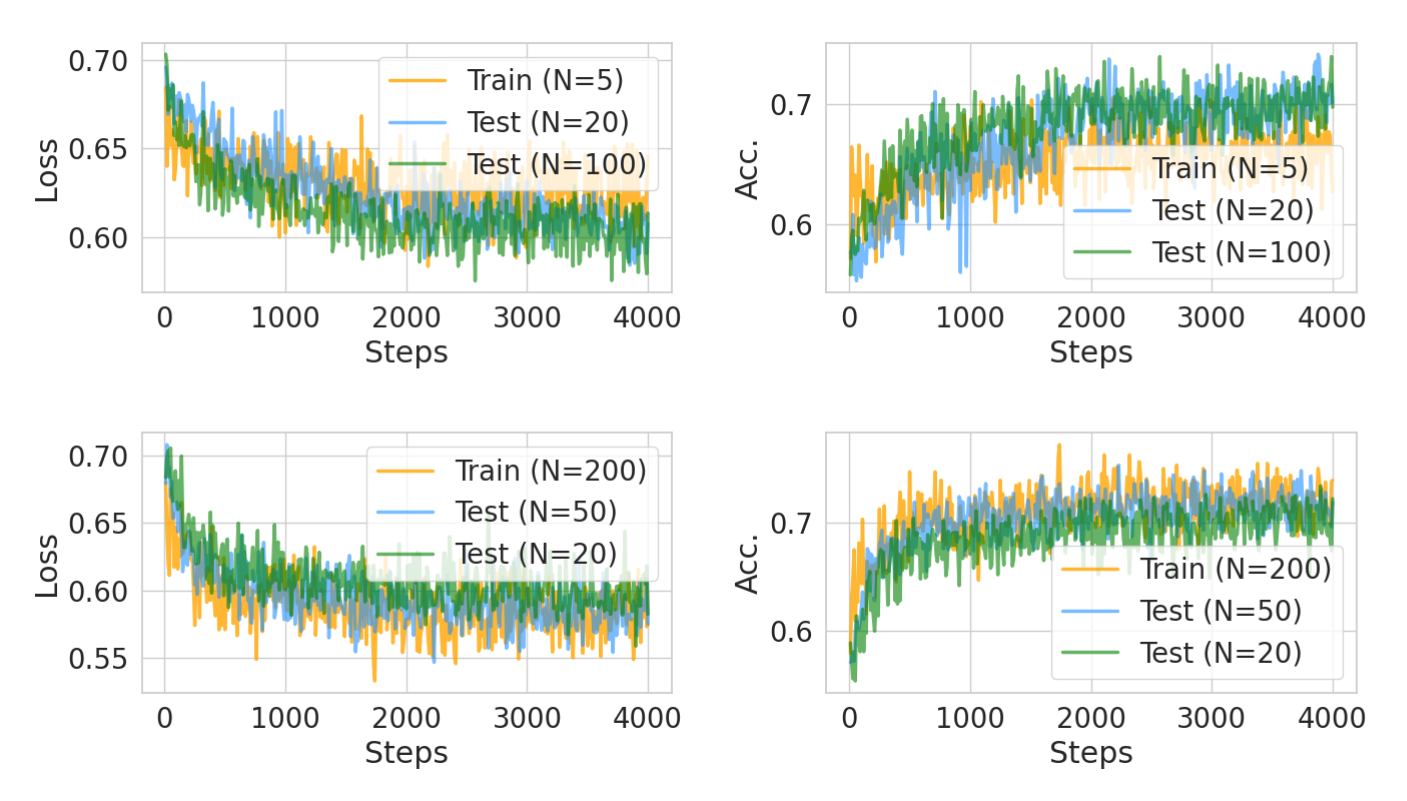}
    \vspace{-2em}
    \caption{\textbf{Convergence result on general graph for $N_{\text{train}} \in \{ 5, 10, 200, 400\}$.
    }
    }
    \label{fig:graphs-convergence-gen-app-left}
\end{figure}

\begin{figure}[t]
    \centering
    \includegraphics[width=\linewidth]{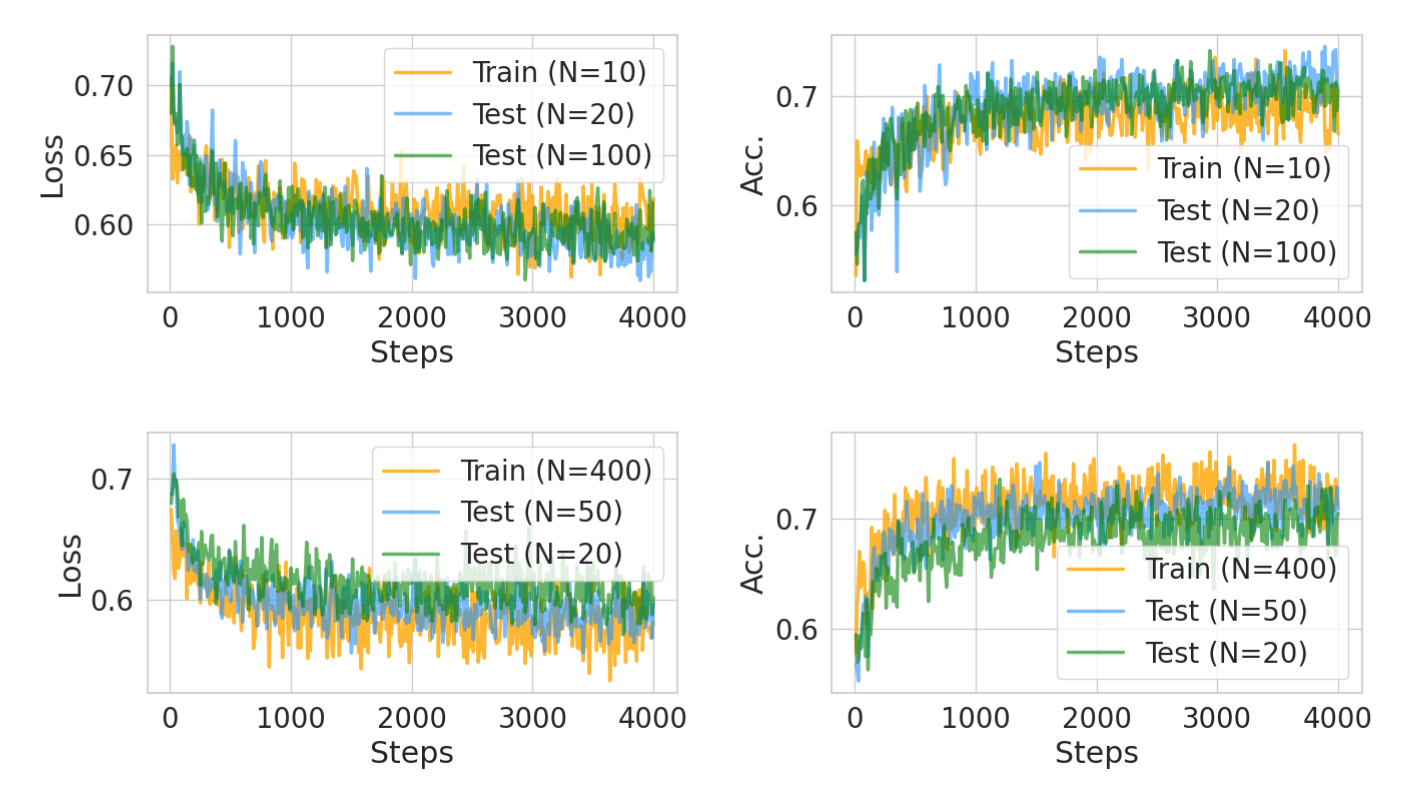}
    \vspace{-2em}
    \caption{\textbf{Left: Convergence result on general graph for $N_{\text{train}} \in \{ 5, 10, 200, 400\}$.
    Right: Convergence result on tree for $N_{\text{train}} \in \{ 5, 10, 200, 400\}$.
    }
    Here we observe an obvious contrast between models trained on different $N_{\text{train}}$.
    For smaller $N_{\text{train}}$, model performance on training dataset is lower than testset.
    For larger $N_{\text{train}}$, we observe the opposite.
    We believe this is due to the fact that smaller $N_{\text{train}}$ does not provide sufficient sample size to recover the probability distribution well.
    }
    \label{fig:graphs-convergence-gen-app-right}
\end{figure}

\paragraph{Tree.}\label{generalization-tree}
The results are demonstrated in  Figure~\ref{fig:tree-convergence-gen} and Figure~\ref{fig:tree-convergence-appendix}.
Overall, we observe that transformers fail to perform MLE based on Bayesian network architecture when $N_{\text{train}}=5$.
However, different from our results on general graph, $N_{\text{train}}=10$ seems to be sufficient for transformers to learn MLE based on Bayesian network architecture.
This result can be explained by the fact that modeling variable 4 only requires to focus on its single parent.
However, in general graphs, some variables have multiple parents, which prevents $N_{\text{train}}=10$ to recover the conditional probability distribution well.

\begin{figure}[h]
    \centering
    \includegraphics[width=\linewidth]{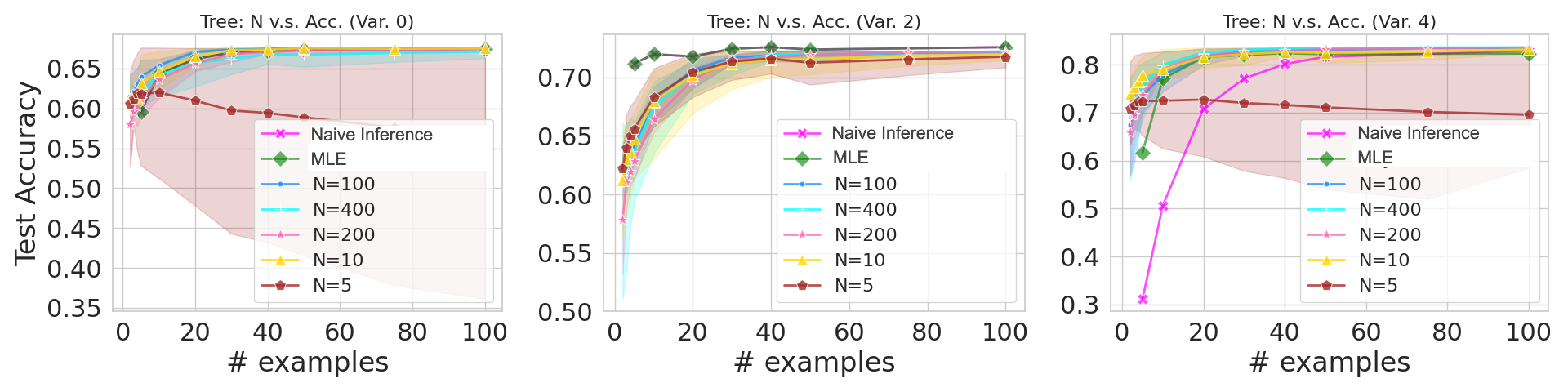}
    \vspace{-2em}
    \caption{\textbf{Generalization Analysis: Inference results on tree.
    }
    Similar to our results on graph, transformers trained on large $N_{\text{train}}$ generalize better than trained on smaller $N_{\text{train}}$.
    Especially with $N_{\text{train}} = 5$, transformers fail to even predict well  
    }
    \label{fig:tree-convergence-gen}
\end{figure}

\begin{figure}[thb]
    \centering
    \includegraphics[width=\linewidth]{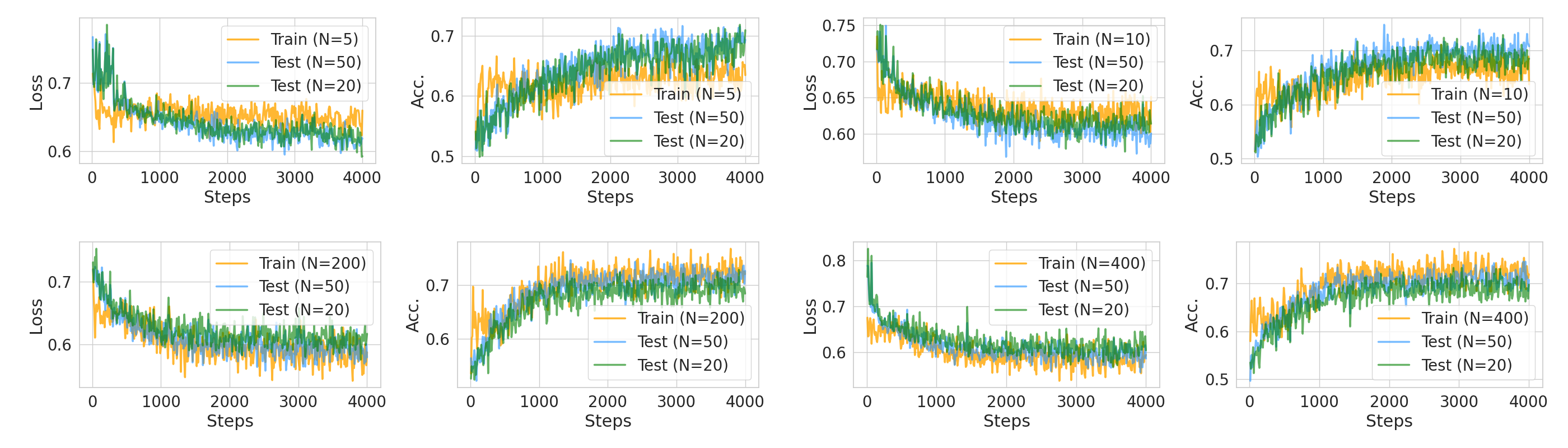}
    \vspace{-2em}
    \caption{\textbf{Generalization Analysis: Convergence result on tree.
    Top: $N_{\text{train}} \in \{ 5, 10\}$,
    Bottom: $N_{\text{train}} \in \{ 200, 400\}$
    }
    Similar to our results on graph, transformers trained on large $N_{\text{train}}$ generalize better than trained on smaller $N_{\text{train}}$.
    The gap between training and testset gets larger close to the end of training.
    }
    \label{fig:tree-convergence-appendix}
\end{figure}

\subsection{Additional Experiment for Categorical Distributions}\label{appendix:add-exp}
Here we conduct experiments on networks with categorical distributions, i.e. the number of possible outcome for each variable is more than 2.
We select the binary tree structure as example, and set the number of possible outcome for each variable as 3.
We report both the test accuracy and test F1 are evaluation metrics, the results are in Figure~\ref{fig:tree-acc-cate} and Figure~\ref{fig:tree-f1-cate}.
As a result, the input dimension of the transformer is $28$.
For all other hyperparameters, we follow Table~\ref{table:hyperparam}.

\begin{figure}[t]
    \centering
    \includegraphics[width=\linewidth]{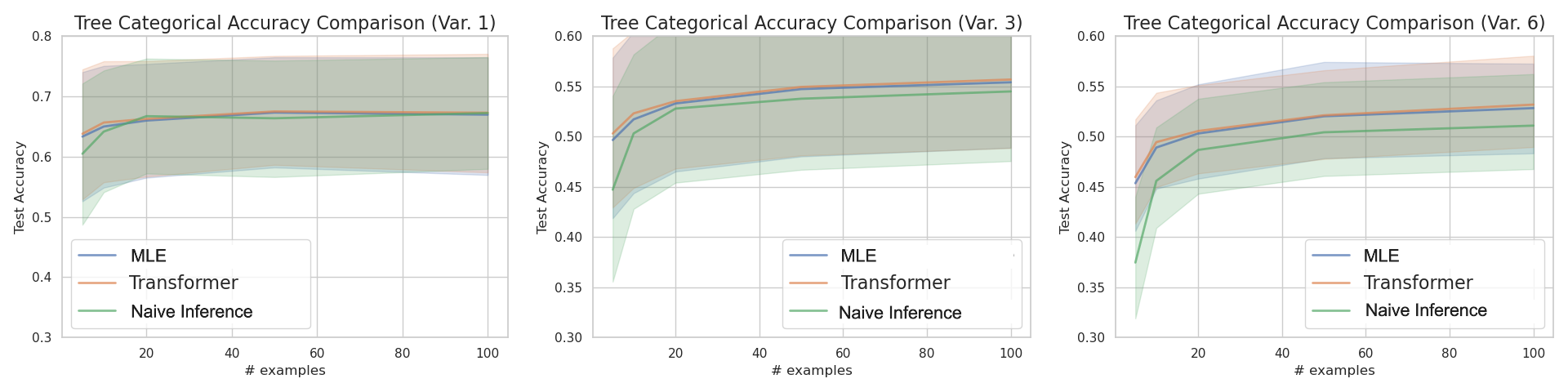}
    \caption{\textbf{Accuracy Comparison for the Tree Network with Categorical distribution.}
    In the figures, we are able to observe the test accuracy follows the same pattern comparing to the ones with binary distribution (Figure~\ref{fig:main-result}).
    The result shows that transformers are capable of learning the network structure and perform MLE based on it.
    }
    \label{fig:tree-acc-cate}
\end{figure}

\begin{figure}[t]
    \centering
    \includegraphics[width=\linewidth]{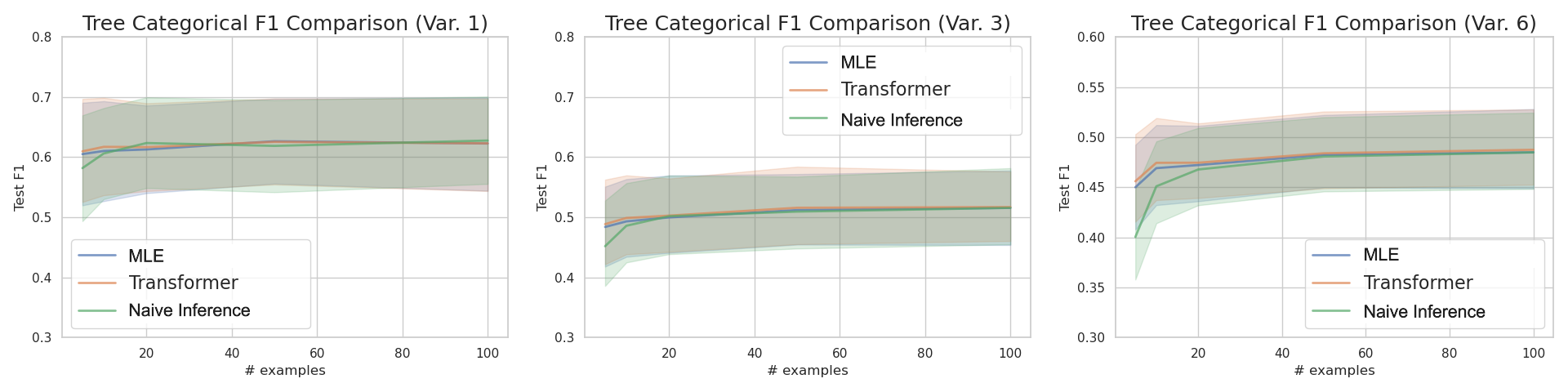}
    \caption{\textbf{F1 Score Comparison for the Tree Network with Categorical distribution.}
    Since we are handling the multi-class prediction, we also report the F1 score for all the baselines.
    Similar to what we observe in the accuracy result, we are also able to observe the test F1 follows the same pattern comparing to the ones with binary distribution (Figure~\ref{fig:main-result}).
    The result again confirms that transformers are capable of learning the network structure and perform MLE based on it.
    }
    \label{fig:tree-f1-cate}
\end{figure}

\section{Experimental Details}\label{appendix:exp-details}

\subsection{Synthetic Data Details}\label{experiment-details:data}
Here we provide visualizations of graphs structures we select in our experiments.
Arrows indicates the causal relationship between variables.
Specifically, the "general graph" contains variables with more than 1 parents, representing a more generalized case.
An interesting design of the general graph is its variable 2 and 3 are both governed by 2 parents.
However, modeling variable 2 can be done via naive inference while modeling variable 3 requires MLE based on Bayesian network architecture, giving us an opportunity to discover such property.
\begin{figure}[t]
    \centering
    \includegraphics[width=0.95\linewidth]{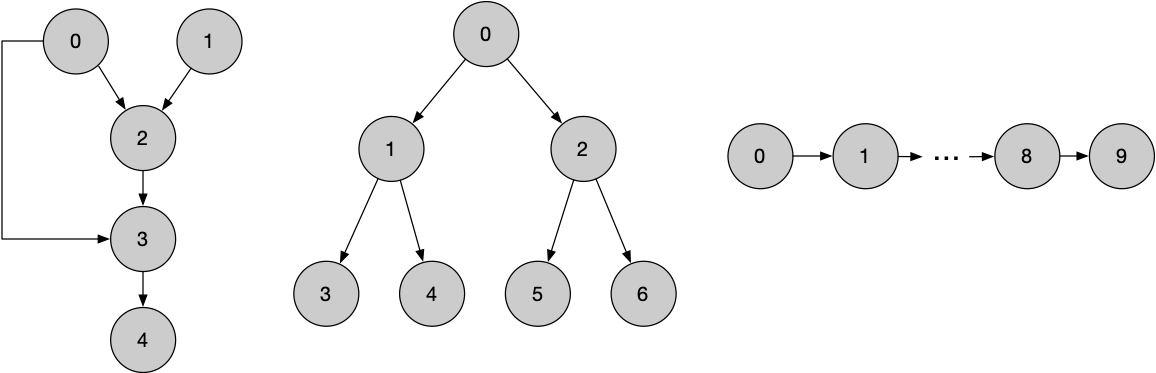}
    \caption{\textbf{Illustrations of graph structures in the experiments.}
    Left to right:
    general graph, tree and chain.
    The curriculum follows the number order of variables.
    Note that for general graph, variable 2, 3 both have 2 parents.
    However, for variable 2, the modeling process is identical for naive inference and MLE based on Bayesian network architecture.
    For variable 3, modeling it is different for naive inference and MLE based on network structure.
    }
    \label{fig:graphs-visualization-app}
\end{figure}

\subsection{Real World Dataset Details}\label{experiment-details:real-data}
For the ACSIncome, we preprocess the features with two major steps:
(1) Remove data points with N/A values.
(2) Merge categories within some dimensions of features.
For (2), the merged features are listed below.
Note that we use the original code name used in the ACSIncome for readers to reference them easily.
\paragraph{SCHL.}
This is a feature indicating individual's education level.
There were 24 categories in this feature before preprocessing, we merged them into 9 categories listed in Table~\ref{tab:educational_attainment}.
For the original categories, please refer to PUMS Documentation.
\begin{table}[h]
\centering
\caption{Merged Categories of SCHL Feature in ACSIncome and ACSPublicCoverage.}
\begin{tabular}{cl}
\toprule
\textbf{Value} & \textbf{Description} \\ \midrule
1            & No Formal Education \\ 
2              & Early Childhood Education \\ 
3              & Elementary School \\ 
4              & Middle School \\ 
4              & High School (Incomplete) \\ 
5              & High School Graduate or Equivalent \\ 
6              & College (No Degree) \\ 
7              & Associate’s Degree \\ 
8              & Bachelor’s Degree \\ 
9              & Advanced Degrees \\ 
\bottomrule
\end{tabular}
\label{tab:educational_attainment}
\end{table}

\paragraph{RELP.}
The RELP feature corresponds to the relationship of the individual to the reference person.
Note that the survey is conducted on household level.
Therefore, non-family residents such as roommates, unmarried partner are also included in this feature.
The merged categories are listed in Table~\ref{tab:relp}.

\begin{table}[h]
\centering
\caption{Merged Categories of RELP Feature in ACSIncome and ACSPublicCoverage.}
\begin{tabular}{cl}
\toprule
\textbf{Value} & \textbf{Description} \\ \midrule
1            & Reference Person \\ 
2              & Immediate Family \\ 
3              & Extended Family \\ 
4              & Non-Family Residents \\ 
4              & Group Quarters Population \\ 
5              & Unknown \\ 
\bottomrule
\end{tabular}
\label{tab:relp}
\end{table}

\paragraph{WKHP.}
This feature indicates the working hour per week of the individual.
The original feature is ranged from 1 to 99, where we categorized them with broader concepts.
The merged categories are listed in Table~\ref{tab:wkhp}.

\begin{table}[h]
\centering
\caption{Merged Categories of WKHP (working hour per week) Feature in ACSIncome and ACSPublicCoverage.}
\begin{tabular}{cl}
\toprule
\textbf{Value} & \textbf{Description} \\ \midrule
1            & No Work \\ 
2              & Part-Time (1 ~ 34 hrs) \\ 
3              & Full-Time Work (35 ~ 48 hrs) \\ 
4              & Overtime Work (49 ~ 98 hrs) \\ 
4              & Extremely High Hours (> 99 hrs) \\ 
5              & Unknown \\ 
\bottomrule
\end{tabular}
\label{tab:wkhp}
\end{table}

\subsection{Training Details}\label{experiment-details}
The hyperparameter table is in Table~\ref{table:hyperparam} and Table~\ref{table:hyperparam-real}.
We ran all experiments on RTX 2080 ti GPUs.
We use PyTorch 1.11 for all models, training and evaluation.
We use AdamW optimizer for training.
For curriculum design, we follow the variable order (index) to reveal variables.
For example, no future variables will be revealed until all of its precedents are revealed during training.
For tree structures, we use BFS to determine the curriculum.
We do not use any learning decay techniques as we find learned transformers perform better without it.
For each training step, we generate sampled $N_{\text{train}}+1$ examples randomly from 1 of our 50k candidate graphs to, to ensure models do not see repetitive data during training.
We log training and test loss every 50 steps, and save the checkpoint with lowest training loss.
For data generation, we use the Python package \texttt{pomegranate} for both constructing networks and sampling.

\subsection{Baselines}\label{experiment-baselines}
Here we explain the baselines used in our experiments.
We use an example for predicting the $M$-th variable of a query sequence $\xb_{1q}, \dots, \xb_{(M-1)q}$ with the first to $(M-1)$-th variable being observed.
Following the setup of in-context learning, we assume a set of $N$ groups context observations $X_{1i}, \dots, X_{Mi}$ for $i = 1, \dots, N$, denoting as $\mathcal{O}$.
Note that the two baselines are not capable of handling unseen features or labels.
Such a case will lead directly to assigning probability $0$ to all categories.

\paragraph{Naive Inference.}
The naive inference method predicts $\xb_{Mq}$ with the following probability distribution.
\[
P(\xb_{Mq} | \xb_{1q}, \dots, \xb_{(M-1)q}, \; \mathcal{O}) 
= 
\frac{\sum_{X_i \in \mathcal{O}} \mathbbm{1} \left(X_{1i} = \xb_{1q}, \dots, X_{(M-1)i} = \xb_{(M-1)q} \right) \cdot \mathbbm{1}(X_{M} = \xb_{Mq})}{\sum_{X_i \in \mathcal{O}} \mathbbm{1} \left(X_{1i} = \xb_{1q}, \dots, X_{(M-1)i} = \xb_{(M-1)q} \right)},
\]
where $\mathbbm{1}$ is the indicator function.

\paragraph{MLE based on True BN.}
For this method, we assume the network structure is known.
Thus, assuming the parents of the $M$-th variable are in the set of $\mathcal{P}$, where $\mathcal{P}_q$ are the parent nodes of $\xb_{Mq}$,
then the MLE method predicts $\xb_{Mq}$ with the following probability distribution.
% \[
% P(\xb_{Mq} | \xb_{1q}, \dots, \xb_{(M-1)q}, \; \mathcal{O}) 
% = 
% \frac{\sum_{X \in \mathcal{O}} \mathbbm{1} \left(X_1 = \xb_{1q}, \dots, X_{M-1} = \xb_{(M-1)q} \right) \cdot \mathbbm{1}(X_{M} = \xb_{Mq})}{\sum_{X \in \mathcal{O}} \mathbbm{1} \left(X_1 = \xb_{1q}, \dots, X_{M-1} = \xb_{(M-1)q} \right)},
% \]
\[
P(\xb_{Mq} | \mathcal{P}_q, \; \mathcal{O}) 
= \frac{\sum_{(X_i, \mathcal{P}_i) \in \mathcal{O}} \mathbbm{1}(\mathcal{P}_i = \mathcal{P}_q) \cdot \mathbbm{1}(X_{M} = \xb_{Mq})}{\sum_{(X_i, \mathcal{P}_i) \in \mathcal{O}} \mathbbm{1}(\mathcal{P}_i = \mathcal{P}_q)},
\]
where $\mathcal{P}_i$ is the parent of $X_{Mi}$.

\begin{table}[h]
        \centering
        \caption{Hyperparameters for Synthetic Data.
        }
        % \resizebox{ \textwidth}{!}{  
        \begin{tabular}{l*{4}{c}}
        \toprule
            \bf{parameter} & Chain & Tree & General \\ 
            \midrule
            optimizer & AdamW & AdamW & AdamW  \\
            steps & 10k & 3k & 2k \\
            learning rate  & $1\text{e-}4$ & $5\text{e-}4$ & $5\text{e-}4$ \\
            weight decay  & $1\text{e-}2$ & $5\text{e-}2$ & $5\text{e-}2$ \\
            batch size  & $64$ & $64$ & $64$  \\
            number of layers & $6$ & $6$ & $6$  \\
            loss function & Cross Entropy& Cross Entropy& Cross Entropy \\
            hidden dimension & $256$ & $256$& $256$ \\
            number of heads  & $8$ & $8$& $8$ \\
            number of examples (Train) $N$ & $100$ & $100$ & $100$  \\
            \bottomrule
        \end{tabular}
        \label{table:hyperparam}
    \end{table}

\begin{table}[h]
        \centering
        \caption{Hyperparameters for ACSIncome.
        }
        % \resizebox{ \textwidth}{!}{  
        \begin{tabular}{l*{1}{c}}
        \toprule
            \bf{parameter} &   \\ 
            \midrule
            optimizer & AdamW   \\
            steps & 40k  \\
            learning rate  & $1\text{e-}4$ \\
            weight decay  & $1\text{e-}2$ \\
            batch size  & $64$  \\
            number of layers & $6$ \\
            loss function & Cross Entropy\\
            hidden dimension & $256$  \\
            number of heads  & $8$  \\
            number of examples (Train) $N$ & $200$ \\
            \bottomrule
        \end{tabular}
        \label{table:hyperparam-real}
    \end{table}

\end{document}